\documentclass[journal]{IEEEtran}
\ifCLASSINFOpdf
\else
\fi
\usepackage[pagebackref=false,breaklinks=true,colorlinks,linkcolor={black},citecolor={black},urlcolor={black},bookmarks=false]{hyperref}
\usepackage{cite}
\usepackage{microtype}
\usepackage{subfigure}
\usepackage{url}

\usepackage{array}
\usepackage[normalem]{ulem}
\usepackage{ulem}
\usepackage{amsmath}
\usepackage{graphicx,subfigure}
\usepackage{longtable}
\usepackage{rotating}
\usepackage{multirow}
\usepackage{makecell}
\usepackage{bm}
\usepackage{float}
\usepackage{flushend}
\usepackage{times}
\usepackage{xcolor}
\usepackage{soul}
\DeclareGraphicsExtensions{.eps}
\usepackage{booktabs}
\usepackage{threeparttable}

\usepackage[ruled]{algorithm2e}
\usepackage{algpseudocode}
\usepackage{tabularx}
\usepackage{amsfonts}
\usepackage{epstopdf}
\usepackage{latexsym}
\usepackage{amssymb}
\newtheorem{proof}{Proof}
\newtheorem{Lemma}{Lemma}
\newtheorem{theorem}{Theorem}
\newtheorem{definition}{Definition}
\def\Sp{{\scriptsize{\textcircled{{\emph{\tiny{\textbf{Sp}}}}}}}}

\makeatletter
\DeclareRobustCommand\onedot{\futurelet\@let@token\@onedot}
\def\@onedot{\ifx\@let@token.\else.\null\fi\xspace}

\def\eg{\emph{e.g}\onedot} 
\def\ie{\emph{i.e}\onedot}

\def\wrt{w.r.t\onedot} 
\def\etal{\emph{et al}\onedot}
\makeatother

\begin{document}
\title{Multiple Graph Learning for Scalable Multi-view Clustering}

\author{Tianyu~Jiang,
        Quanxue~Gao,% <-this % stops a space
        and~Xinbo~Gao,~\IEEEmembership{Senior,~IEEE}
\thanks{T. Jiang and Q. Gao are with the State Key Lab. of Integrated Services Networks, Xidian University, Xi'an 710071, China (e-mail: qxgao@xidian.edu.cn).} % 
\thanks{X. Gao is with the School of Electronic Engineering, Xidian University, Xi’an 710071, China and with the Chongqing Key Laboratory of Image Cognition, Chongqing University of Posts and Telecommunications, Chongqing 400065, China.}
\thanks{Manuscript received XXX XX, XXXX; revised XXX XX, XXXX.}}

\markboth{Journal of \LaTeX\ Class Files,~Vol.~XX, No.~XX, June~2021}%
{Shell \MakeLowercase{\textit{Jiang et al.}}: Bare Demo of IEEEtran.cls for IEEE Journals}

\maketitle

\begin{abstract}
Graph-based multi-view clustering has become an active topic due to the efficiency in characterizing both the complex structure and relationship between multimedia data. However, existing methods have the following shortcomings: (1) They are inefficient or even fail for graph learning in large scale due to the graph construction and eigen-decomposition. (2) They cannot well exploit both the complementary information and spatial structure embedded in graphs of different views. To well exploit complementary information and tackle the scalability issue plaguing graph-based multi-view clustering, we propose an efficient multiple graph learning model via a small number of anchor points and tensor Schatten $p$-norm minimization. Specifically, we construct a hidden and tractable large graph by anchor graph for each view and well exploit complementary information embedded in anchor graphs of different views by tensor Schatten $p$-norm regularizer. Finally, we develop an efficient algorithm, which scales linearly with the data size, to solve our proposed model. Extensive experimental results on several datasets indicate that our proposed method outperforms some state-of-the-art multi-view clustering algorithms.
\end{abstract}


\IEEEpeerreviewmaketitle

\section{Introduction}
\IEEEPARstart{W}{ith} rapid development of the Internet, we can collect massive unlabeled multimedia data such as images, videos, audio, and documents, these data are usually represented by different features, which are called multi-view data. Multi-view clustering partitions a group of unlabeled multimedia data into different groups according to the relationships between data such that the data in the same group have high similarity to each other, while data in different groups are diverse. Due to the fact that graph is an important data structure which intuitively represents the relationship between data, graph based multi-view clustering has achieved impressive performance and become an active topic in computer vision and artificial intelligence. One of the most representative graph-based clustering methods is co-regularized multi-view spectral clustering (Co-reg)~\cite{2011Co}. It performances spectral clustering on each view to get the corresponding indicator matrix, and then gets the common indicator matrix by minimizing divergence between indicator matrices of different views, finally discrete labels are obtained by performing $K$-means on the common indicator matrix.

The performance of Co-reg degrades remarkably when the quality of the predefined graph is poor. In real applications, it is very difficult to manually construct a reasonable graph for each view. Thus, the key problem of graph-based clustering is how to adaptively learn a good graph. In order to work on solving this problem, many effective methods have been proposed to adaptively learn graphs for multi-view clustering~\cite{SFMC,ref_MLAN,NieLL17,XuHNL17,2017Graph}, one of the most representative methods is MLAN~\cite{ref_MLAN}. It sorts manifold learning to adaptively learn the common shared graph for all views and presents multi-view learning with adaptive neighbors (MLAN).

Although the learned graph by MLAN well explores the local intrinsical structure of data, the main important problem plaguing MLAN and other graph-based methods is that, they are inefficient or even fail in large scale due to the computational intensiveness of full-size graph construction and eigen-decomposition of Laplacian matrix. The computational complexity of MLAN is $\bm{\mathcal{O}}(N^2K)$ for graph construction and $\bm{\mathcal{O}}(N^3)$ for eigen-decomposition of Laplacian matrix, respectively, where $N$ and $K$ denote the number of data points and clusters, respectively. The computational issue becomes critical when $N$ is large in the era of big data. Furthermore, the graph learned by MLAN cannot well exploit the relationship among data. It assumes that the same sample in different views has the same neighbors, which cannot be satisfied due to the fact that each view contains some content of the objects that cannot be included in other views. Thus, the neighbors of data from different views should be different. Unfortunately, MLAN neglect this fact. Finally, it can not well exploit the complementary information embedded in similarity graphs of different views. MLAN assumes that similarity graphs of different views, which characterize the similarity between data of the corresponding view, have the same value, it indicates that all the views have the same relationships between data. This requirement is unrealistic and very strict in practical applications due to the fact that each view usually has different effect for clustering.

Drawing inspiration from anchor graph~\cite{ZhuL05,LiuHC10,SFMC} and the recently proposed tensor nuclear norm~\cite{XiaZGSHG21,GZXTPAMI2020}, which helps exploit the complementary information of different views, we present a multiple graph learning for scale multi-view clustering. To be specifical, for each view, our model employs $N\times M$ anchor graph, which characterizes the relationship between $N$ samples and $M$ anchors, to construct a hiddenly full size $N\times N$ graph, where $M\ll N$. To well exploit the complementary information embedded in anchor graphs of different views, we minimize the divergence between anchor graphs by tensor Schatten $p$-norm regularizer. Finally, an efficient algorithm is presented to solve our model.The main contributions are summarized as follows:

\begin{itemize}
  \item We adaptively learn similarity graph for each view and minimize the divergence between graphs of different views by tensor Schatten $p$-norm minimization which well exploits the complementary information embedded in graphs of different views.
  \item The common shared graph, which is learned by our model, has the exact $K$-connected components which characterizes the cluster structure of data. Thus, our model can directly get clustering labels for data without post-processing.
  \item Our method employs $N\times M$ ($M\ll N$) anchor graph to hiddenly construct full size, \ie, $N\times N$ graph. Thus, the computational complexity reduces from $\bm{\mathcal{O}}(N^2K)$ to $\bm{\mathcal{O}}(NMd)$ for each view, where $d$ is the number of feature dimensional of data matrix.
  \item Our proposed algorithm scales linearly with the data size. Thus, it is efficient for graph learning in large scale data. Extensive experiments indicate the effectiveness of our proposed method in several databases.
\end{itemize}

\section{Related work}
Being the efficiency of characterizing the relationship between data, graph based clustering has becoming one of the most representative clustering techniques. One of the most representative methods is spectral clustering, which employs spectral graph theory and formulates the clustering task as an eigen-decomposition problem. The core of spectral clustering is to find an optimal partition of a graph into several subgraphs according to a criterion with some criterion, such as graph cut, normalized graph-cut. For multi-view clustering, one of the most representative spectral clustering methods is co-regularized (Co-reg) multi-view spectral clustering~\cite{KumarRD11r}. It gets the low-dimensional embedding for each view by eigen-decomposition and learns the common shared formulates by minimizing the divergence between low-dimensional embedding of different views.

Although Co-reg has good performance, the treat all views equally, which does not make sense in real-word applications. To improve robustness of clustering algorithm, Nie \etal~\cite{NieLL16} adaptively assigned the weighted values for different views and developed auto-weighted graph learning method (AMGL).

However, the graphs in their methods are fixed. When the input graphs are of poor quality, their clustering performance degrades remarkably. To solve this problem, the multi-view learning with adaptive neighbours (MLAN)~\cite{ref_MLAN} and the self-weighted multi-view clustering (SwMC)~\cite{NieLL17} are proposed. To capture the shared information among the graphs of multiple views, Xia \etal proposed the robust multi-view spectral clustering (RMSC)~\cite{XiaPDY14}. To well characterizes high-order information embedded in view-similar graph, Wu \etal~\cite{WuLZ19} proposed essential tensor learning for multi-view spectral clustering (ETLMSC) method. Although they have achieved good results, they are all expensive time-consuming duo to the computation of the $N\times N$ graph ($N$ is the data size) and fail to deal with the large-scale datasets.

To reduce the computation complexity, some fast algorithms have been developed. For example, Li \etal employed local manifold fusion to integrate heterogeneous features and proposed a large-scale multi-view spectral clustering method (MVSC)~\cite{LiNHH15}. Hu \etal integrated nonnegative embedding and spectral embedding into a unified framework, and proposed multi-view spectral clustering method (SMSC)~\cite{HuNWL20}. Xu \etal proposed the re-weighted multiple k-means (RDEKM)~\cite{XuHNL17} for large-scale data clustering. Li \etal proposed the bipartite graph fusion based multi-view clustering method (SFMC)~\cite{SFMC}. Although the aforementioned methods can explore large-scale multi-view data, all of them neglect the spatial structure and the complementary information embedded in multi-view data, resulting in inferior clustering results. Comparatively, our proposed method learns graph by minimizing tensor Schatten $p$-norm on the view-similarities of all views. Extensive experiments indicate the efficiency of our method.

\textbf{\emph{Notations:}} Throughout the paper, we use
bold calligraphy letters for third-order tensors, \eg, ${\bm{\mathcal {Z}}} \in{\mathbb{R}} {^{{n_1} \times {n_2} \times {n_3}}}$, bold upper case letters for matrices, \eg, ${\bf{Z}}$, bold lower case letters for vectors, \eg, ${\bf{z}}$, and lower case letters such as ${Z_{ijk}}$ for the entries of ${\bm{\mathcal {Z}}}$. Moreover, the $i$-th frontal slice of ${\bm{\mathcal {Z}}}$ is ${\bm{\mathcal {Z}}}^{(i)}$. $\overline {{\bm{\mathcal {Z}}}}$ is the discrete Fast Fourier Transform (FFT) of ${\bm{\mathcal {Z}}}$ along the third dimension, \ie, $\overline {{\bm{\mathcal {Z}}}} = fft({{\bm{\mathcal Z}}},[ ],3)$. Thus, $\bm{{\mathcal Z}} = ifft({\overline {\bm{\mathcal Z}}},[ ],3)$. $tr(\mathbf{Z})$ denotes the trace of matrix $\mathbf{Z}$. $\mathbf{I}$ is an identity matrix. ${\bf{Z}}^\textrm{T}$ is defined as the transpose of matrix  ${\bf{Z}} $.

\section{Methodology}
Given a multi-view dataset $\{{\mathbf{X}}^{(1)},{\bf{X}}^{(2)}, \cdots,{\bf{X}}^{(V)}\}$ with $K$ clusters, where $ {\bf{X}}^{(v)} \in{\mathbb{R}}^{N \times d_v}$ is the data matrix of $v$-th view, $d_v$ is the dimension of data matrix in $v$-th view, $N$ and $V$ are the number of data points and views, respectively. The objective of MLAN is
\begin{equation}\label{obj-jty-1}
\begin{aligned}
&\mathop {\min }\limits_{\bf{S}} \sum\limits_{v = 1}^V {\sqrt {\sum\limits_{i=1}^N \sum\limits_{j = 1}^N {\left\| {{\bf{x}}_i^{(v)} - {\bf{x}}_j^{(v)}} \right\|_2^2{S_{ij}}} } }  + \alpha \left\| {\bf{S}} \right\|_F^2\\
\textrm{s.t.}&\quad\quad \forall i,{\bf{s}}_i^{\textrm{{T}}}{\bf{1}} = 1, 0 \le {S_{ij}} \le 1, \textrm{rank}({{\bf{L}}_{\bf{S}}}) = N-K
\end{aligned}
\end{equation}
where, ${\bf{S}} \in {\mathbb{R}}^{N \times N}$ is the similarity matrix, and ${{\bf{L}}_{\bf{S}}} = {{\bf{D}}_{\bf{S}}} - \frac{{{\bf{(}}{{\bf{S}}^{\textrm{{T}}}} + {\bf{S)}}}}{2}$ is the Laplacian matrix, ${{\bf{D}}_{\bf{S}}}$ is a diagonal matrix, and ${{\bf{D}}_{\bf{S}}}(i,i)=\sum\nolimits_j {\frac{{({S_{ij}} + {S_{ji}})}}{2}}$, $\alpha$ is a trade-off parameter.

Despite the promising preliminary results for clustering, the model (\ref{obj-jty-1}) generally suffers from high computational complexity, and fails to handle large-scale multi-view data in real-world situations. To be specific, the main time of model (\ref{obj-jty-1}) is spent on both the full-size similarity matrix calculation and eigen-decomposition of Laplacian matrix $\mathbf{L}_\mathbf{S} \in \mathbb{R}^{N \times N}$, the computational complexity required for which are $\mathcal{O}({N^2}K)$ and $\mathcal{O}(N^3)$, respectively. Moreover, the model (\ref{obj-jty-1}) cannot well exploit the complementary information and spatial structure embedded in similarity matrices of different views. Therefore, the learned the similarity matrix in the model (\ref{obj-jty-1}), which is shared by different views, cannot well characterize the cluster structure, resulting in inferior clustering results.

To address the aforementioned issues, we propose a multiple graph learning model for scalable multi-view clustering. To be specific, in order to reduce the computational complexity of the model (\ref{obj-jty-1}), inspired by the remarkable success of the bipartite graph~\cite{ZhuL05,LiuHC10,SFMC}, we leverage an effective bipartite graph $\mathbf{Z}\in {\mathbb{R}}^{N \times M}$ to replace the full-size similarity matrix $\mathbf{S}$, where the newly introduced bipartite graph $\mathbf{Z}$ describes the relationship between $N$ data points and $M$ anchors. What's more, to well characterize the cluster structure and directly obtain the clustering results without post-processing, motivated by the Lemma~\ref{lem1}, we introduce the Laplacian rank constraint to ensure the learned graph $\mathbf{Z}$ has exact $K$-connected components. Thus, we have
\begin{equation}\label{obj-jty-2}
\begin{aligned}
&\mathop {\min }\limits_{\bf{Z}} \sum\limits_{v = 1}^V {\sqrt {\sum\limits_{i = 1}^N {\sum\limits_{j = 1}^M {\left\| {{\bf{x}}_i^{(v)}{\bf{ - a}}_j^{(v)}} \right\|_2^2{Z_{ij}}} } } {\bf{ + }}\alpha \left\| {\bf{Z}} \right\|_F^2} \\
\textrm{s.t.}&\quad \forall i, {\bf{z}}_i^\textrm{T}{\bf{1}} = 1,0 \le {Z_{ij}} \le 1, \textrm{rank}{\bf{(}}{\widetilde {\bf{L}}_{{{\bf{S}}}}}{\bf{) = }}N + M - K
\end{aligned}
\end{equation}
where $ {\widetilde {\bf{L}}_{{{\bf{Z}}}}}{\bf{ = I - D}}_{{}}^{ - {\textstyle{1 \over 2}}}{{\bf{B}}}{\bf{D}}_{{}}^{ - {\textstyle{1 \over 2}}}$ is the normalized Laplacian matrix of matrix ${{\bf{B}}} \in {\mathbb{R}}{^{(N + M) \times (N + M)}}$, and ${{\bf{B}}}{\bf{ = }}\left[ {\begin{array}{*{20}{c}}
{}&{{{\bf{Z}}}}\\{{{{\bf{(}}{{\bf{Z}}}{\bf{)}}}^{\textrm{{T}}}}}&{} \end{array}} \right]$,
${{\bf{D}}_{{}}}$ is a diagonal matrix whose diagonal elements are ${{\bf{D}}_{{}}}(i,i){\bf{ = }}\sum\nolimits_{j = 1}^{N + M} {P_{ij}}$, ${\bf{a}}_j^{(v)}$ is the $j$-th anchor in the $v$-th view.

\begin{Lemma}~\cite{LL1}\label{lem1} The multiplicity $K$ of the eigenvalue zeros of $\widetilde{\mathbf{L}}_{\mathbf{Z}}$ is equals to the number of connected components in the graph associated with ${\mathbf{Z}}$.
\end{Lemma}

It is difficult to directly solve the Laplacian rank constraint, \ie, $\textrm{rank}{\bf{(}}{\widetilde {\bf{L}}_{{{\bf{S}}}}}{\bf{)}} = N + M - K$. To this end, according to Ky Fan's theorem~\cite{KFF}, the Laplacian rank constraint can be approximated by solving the following model:
\begin{equation}\label{obj-jty-3}
\begin{aligned}
\mathop {\min }\limits_{{{\bf{F}}^{\textrm{T}}}{{\bf{F}}} = {\bf{I}}} {tr\left({\bf{F}}^{\textrm{T}}{{\widetilde {\bf{L}}}_{{{\bf{Z}}}}}{\bf{F}}\right)}
\end{aligned}
\end{equation}
where, ${{\bf{F}}} = [{{\bf{f}}_1}; \cdots ;{\bf{f}}_{N + M}] \in {\mathbb{R}}{^{(N + M) \times K}}$ is the indicator matrix.

Then, the model (\ref{obj-jty-2}) can be rewritten as
\begin{equation}\label{obj-jty-4}
\begin{aligned}
&\mathop {\min }\limits_{\bf{Z}, \mathbf{F}} \sum\limits_{v = 1}^V {\sqrt {\sum\limits_{i = 1}^N {\sum\limits_{j = 1}^M {\| {{\bf{x}}_i^{(v)}{\bf{ - a}}_j^{(v)}} \|_2^2{Z_{ij}}} } } {\bf{ + }}\alpha \| {\bf{Z}} \|_F^2} + \beta {tr({\bf{F}}^{\textrm{T}}{{\widetilde {\bf{L}}}_{{{\bf{Z}}}}}{\bf{F}})} \\
&\quad\quad\quad\textrm{s.t.}\quad \forall i, {\bf{z}}_i^\textrm{T}{\bf{1}} = 1,0 \le {Z_{ij}} \le 1, {{\bf{F}}^{\textrm{{T}}}}{\bf{F}} = {\bf{I}}
\end{aligned}
\end{equation}
where, $\beta$ is a hidden parameter, and it can be adaptively updated as follows. We first initialize $\beta$ with a small value, and update it according to the number of eigenvalue zero of ${\widetilde {\bf{L}}_{{{\bf{Z}}}}}$ after each iteration. If this number is smaller than $K$, $\beta$ is multiplied by 2, or if it is greater than $K+1$, $\beta$ is divided by 2, otherwise we terminate the iterations.

For multi-view data,  each view contains some content of the objects that cannot be included in other views. Thus, the graphs, which characterize of the relationship among data points and anchors in different views, should be different. However, the model (\ref{obj-jty-4}) ignores this. To better exploit the cluster structure, we rewrite the model (\ref{obj-jty-4}) as
\begin{equation}\label{obj-jty-6}
\begin{aligned}
&\mathop {\min }\limits_{\mathbf{{Z}}, \mathbf{F}} \sum\limits_{v = 1}^V \left({\sqrt {\sum\limits_{i = 1}^N {\sum\limits_{j = 1}^M {\| {{\bf{x}}_i^{(v)}{\bf{ - a}}_j^{(v)}} \|_2^2{Z_{ij}^{(v)}}} } } {\bf{ + }}\alpha \| {\mathbf{{Z}}^{(v)}} \|_F^2}\right) + \beta {tr({\bf{F}}^{\textrm{T}}{{\widetilde {\bf{L}}}_{{{\bf{Z}}}}}{\bf{F}})} \\
&\quad\quad\quad\textrm{s.t.}\quad \forall i, ({\mathbf{{z}}_i^{(v)}})^\textrm{T}{\bf{1}} = 1,0 \le {Z_{ij}^{(v)}} \le 1, {{\bf{F}}^{\textrm{{T}}}}{\bf{F}} = {\bf{I}}
\end{aligned}
\end{equation}
where $\mathbf{Z}={\frac{1}{V}}\sum\limits_{v = 1}^V {\mathbf{Z}}^{(v)}$ is the graph shared by different views.

To better exploit both the complementary information and spatial structure embedded in graphs of different views, motivated by recently proposed tensor nuclear norm~\cite{XiaZGSHG21,GZXTPAMI2020}, we utilize the tensor Schatten $p$-norm (See Definition~\ref{definitionTSP}) to minimize differences between graphs.
\begin{definition}\label{definitionTSP}
Given ${\bm{\mathcal {Z}}} \in{\mathbb{R}} {^{{n_1} \times {n_2} \times {n_3}}}$, $h ={\min ({n_1},{n_2})}$, tensor Schatten $p$-norm of tensor ${\bm{\mathcal {Z}}}$ is defined as
\begin{equation}\label{eq3}
\begin{aligned}
{\| \bm{{\mathcal Z}} \|_\Sp}=(\sum\limits_{i = 1}^{n3} {\| {{{{{\overline {\bm{\mathcal {Z}}}}}}^{(i)}}} \|_\Sp^p{)^{\frac{1}{p}}}} = ({\sum\limits_{i = 1}^{n3} {\sum\limits_{j = 1}^{h} {{\sigma _j}{{({{{{\overline {\bm{\mathcal {Z}}}}}}^{(i)}})^p}}} )^{\frac{1}{p}}}}
\end{aligned}
 \end{equation}
where, $0<p\leq 1$, ${\sigma _j}({{{\overline {\bm{\mathcal {Z}}}}}^{(i)}})$ denotes the $j$-th singular value of ${{{\overline {\bm{\mathcal {Z}}}}}^{(i)}}$.
\end{definition}

Finally, our objective function is
\begin{equation}\label{obj-jty-7}
\begin{aligned}
&\mathop {\min }\limits_{\mathbf{{Z}}, \mathbf{F}} \sum\limits_{v = 1}^V \left({\sqrt {\sum\limits_{i = 1}^N {\sum\limits_{j = 1}^M {\| {{\bf{x}}_i^{(v)}{\bf{ - a}}_j^{(v)}} \|_2^2{Z_{ij}^{(v)}}} } } {\bf{ + }}\alpha \| {\mathbf{{Z}}^{(v)}} \|_F^2}\right) + \lambda\| \bm{\mathcal{Z}} \| _{\Sp}^p\\
&\quad\quad\quad\quad+\beta {tr({\bf{F}}^{\textrm{T}}{{\widetilde {\bf{L}}}_{{{\bf{Z}}}}}{\bf{F}})} \\
&\quad\quad\quad\textrm{s.t.}\quad \forall i, ({\mathbf{{z}}_i^{(v)}})^\textrm{T}{\bf{1}} = 1,0 \le {Z_{ij}^{(v)}} \le 1, {{\bf{F}}^{\textrm{{T}}}}{\bf{F}} = {\bf{I}}
\end{aligned}
\end{equation}
where, ${\bm{{\mathcal Z}}}(:,v,:) = {{\bf{Z}}^{(v)}}$.

\subsection{Optimization}
To solve the model (\ref{obj-jty-7}), we adopt Augmented La-grange Multiplier~\cite{2011Linearized}, and introduce an auxiliary variable $\mathcal{J}$. Thus,  the model (\ref{obj-jty-7}) can be rewritten as
\begin{equation}\label{obj-jty-10}
\begin{aligned}
\bm{{\mathcal L}}&({{\bf{Z}}^{(1)}}, \cdots,{{\bf{Z}}^{(V)}}, \bm{{\mathcal J}}, {\bf{F}})\\
&=\sum\limits_{v = 1}^V \left({\sqrt {\sum\limits_{i = 1}^N {\sum\limits_{j = 1}^M {\| {{\bf{x}}_i^{(v)}{\bf{ - a}}_j^{(v)}} \|_2^2{Z_{ij}^{(v)}}} } } {\bf{ + }}\alpha \| {\mathbf{{Z}}^{(v)}} \|_F^2}\right) + \lambda\| \bm{\mathcal{J}} \| _{\Sp}^p\\
&\quad\quad+\beta {tr({\bf{F}}^{\textrm{T}}{{\widetilde {\bf{L}}}_{{{\bf{Z}}}}}{\bf{F}})}+\langle {\mathcal{Y}}{\bf{,}}{\mathcal{Z}}{\bf{ - }}{\mathcal{J}}\rangle {\bf{ + }}\frac{\mu }{2}\left\| {{\mathcal{J}} - {\mathcal{Z}}} \right\|_F^2\\
&\quad\textrm{s.t.}\quad \forall i, ({\mathbf{{z}}_i^{(v)}})^\textrm{T}{\bf{1}} = 1,0 \le {Z_{ij}^{(v)}} \le 1, {{\bf{F}}^{\textrm{{T}}}}{\bf{F}} = {\bf{I}}
\end{aligned}
\end{equation}
where, ${\bm{{\mathcal Y}}}(:,v,:) = {{\bf{Y}}^{(v)}}$ is the Lagrange multiplier, $\mu  > 0$ is the adaptive penalty factor. Consequently, the optimization
process could be separated into the following three steps:

$\bullet$ \textbf{Solving $\textbf{F}$ with fixed $\mathbf{Z}^{(v)}$ and $\mathcal{{J}}$}. In this case, the optimization problem \wrt $\textbf{F}$ in the model (\ref{obj-jty-10}) becomes
\begin{equation}\label{obj-jty-12}
\begin{aligned}
\mathbf{F}^{\ast}&=\mathop {\arg \min }\limits_{{{\bf{F}}^\mathbf{T}}{\bf{F}} = {\bf{I}}} tr(\mathbf{F}^\mathbf{T}\widetilde{\mathbf{L}}_{\mathbf{Z}}\mathbf{F})\\
\end{aligned}
\end{equation}
where $ {\widetilde {\bf{L}}_{{{\bf{Z}}}}}{\bf{ = I - D}}_{{}}^{ - {\textstyle{1 \over 2}}}{{\bf{B}}}{\bf{D}}_{{}}^{ - {\textstyle{1 \over 2}}}$.

If we directly solve the model (\ref{obj-jty-12}), it requires $\bm{\mathcal{O}}((N+M)^2K)$ computational complexity, which is squared with the data size $N$. To this end, we provide a time-economical algorithm. Substituting ${\widetilde {\bf{L}}_{{{\bf{Z}}}}}{\bf{ = I - }}{{\bf{D}}^{ - \frac{1}{2}}}\left[ {\begin{array}{*{20}{c}}
{}&{\bf{Z}}\\
{{{\bf{Z}}^{\textrm{{T}}}}}&{}
\end{array}} \right]{{\bf{D}}^{ - \frac{1}{2}}}$ into the model (\ref{obj-jty-12}), and by simple algebra, we have
\begin{equation}\label{P-3}
\begin{aligned}
tr(\mathbf{F}^\textrm{T}\widetilde{\mathbf{L}}_{(\mathbf{Z})}\mathbf{F})&= tr(\mathbf{F}^\textrm{T}\mathbf{F})-tr(\mathbf{F}^\textrm{T}\mathbf{D}^{-\frac{1}{2}}\mathbf{P}\mathbf{D}^{-\frac{1}{2}}\mathbf{F})\\
&=\textrm{Constant}-tr(\mathbf{F}^\textrm{T}\mathbf{D}^{-\frac{1}{2}}\mathbf{P}\mathbf{D}^{-\frac{1}{2}}\mathbf{F})\\
\end{aligned}
\end{equation}

Denote by ${\bf{F = }}\left[ {\begin{array}{*{20}{c}}
{\bf{P}}\\
{\bf{Q}}
\end{array}} \right] \in{\mathbb{R}} {^{(N + M) \times K}},{\rm{ }}{\bf{D = }}\left[ {\begin{array}{*{20}{c}}
{{{\bf{D}}_P}}&{}\\
{}&{{{\bf{D}}_Q}}
\end{array}} \right]$, where ${\bf{P}}\in \mathbb{R}^{N\times K}$ is the first $N$ rows of ${\bf{F}}$, and ${\bf{Q}}\in \mathbb{R}^{M\times K}$ is the rest $M$ rows of ${\bf{F}}$. ${{\bf{D}}_P}\in \mathbb{R}^{N\times N}$ and ${{\bf{D}}_Q}\in \mathbb{R}^{M\times M}$ are respectively the left-up block and right-bottom block of ${\bf{D}}$, whose diagonal elements are $\mathbf{D}_P(i,i)=\sum\nolimits_{j=1}^M{{\bf{Z}}^{(v)}(i,j)}$ and $\mathbf{D}_Q(j,j)=\sum\nolimits_{i=1}^N{{\bf{Z}}(i,j)}$. Substituting them and Eq. (\ref{P-3}) into the model (\ref{obj-jty-12}), and by simple algebra, the optimal solution of the model (\ref{obj-jty-12}) can be obtained by solving
\begin{equation}\label{obj-jty-13}
\begin{aligned}
&\mathop {\max }\limits_{{\bf{F}} \in {\mathbb{R}} {^{(N + M) \times K}}} tr{\bf{(}}{{\bf{P}}^{\textrm{{T}}}}\mathbf{H}{\bf{Q)}}\\
&\quad\textrm{s.t.}\quad{\rm{  }}{{\bf{P}}^{\textrm{{T}}}}{\bf{P + }}{{\bf{Q}}^{\textrm{{T}}}}{\bf{Q}} = {\bf{I}}
\end{aligned}
\end{equation}
where $\mathbf{H}={\bf{D}}_P^{ - \frac{1}{2}}{\bf{ZD}}_Q^{ - \frac{1}{2}}$. To solve the model (\ref{obj-jty-13}), we first introduce Theorem~\ref{lem2}
\begin{theorem}\label{lem2} Given $\mathbf{H}\in \mathbb{R}^{N\times M}$, $\mathbf{P}\in \mathbb{R}^{N\times K}$, $\mathbf{Q}\in \mathbb{R}^{M\times K}$. The optimal solutions of
\begin{equation}\label{P-1-1}
\max_{{\mathbf{P}^{\mathbf{T}}}{\mathbf{P}} + {\mathbf{Q}^{\textrm{{T}}}}{\mathbf{Q}} = {\bf{I}}} tr(\mathbf{P}^{\textrm{{T}}}\mathbf{H}\mathbf{Q})
\end{equation}
are $\mathbf{P}=\frac{\sqrt{2}}{2}\mathbf{U}_1$, $\mathbf{Q}=\frac{\sqrt{2}}{2}\mathbf{V}_1$, where $\mathbf{U}_1$ and $\mathbf{V}_1$ are the leading $K$ left and right singular vectors of $\mathbf{H}$, respectively.
\end{theorem}

\begin{proof}\label{PRO1}
According to Eq. (\ref{obj-jty-13}), we have
\begin{equation}\label{nn1}
\begin{aligned}
tr&(\mathbf{P}^{\mathbf{T}}\mathbf{H}\mathbf{Q} )=\frac{1}{2}(tr(\mathbf{P}^{\textrm{{T}}}\mathbf{H}\mathbf{Q} + tr(\mathbf{Q}^{\textrm{{T}}}\mathbf{H}^{\textrm{{T}}}\mathbf{P}))\\
&=\frac{1}{2}tr\left({\left[ {\begin{array}{*{20}{c}}
{{{\bf{P}}}}\\
{{{\bf{Q}}}}
\end{array}} \right]^{\textrm{{T}}}}\left[ {\begin{array}{*{20}{c}}
{}&{\bf{H}}\\
{{{\bf{H}}^{\textrm{{T}}}}}&{}
\end{array}} \right]\left[ {\begin{array}{*{20}{c}}
{{{\bf{P}}}}\\
{{{\bf{Q}}}}
\end{array}} \right]\right)
\end{aligned}
\end{equation}
Then, the problem (\ref{obj-jty-13}) is equal to
\begin{equation}\label{nnM}
\begin{aligned}
&\mathop {\arg\max }\limits_{{{{\bf{P}}},{{\bf{Q}}}}} \frac{1}{2}tr\left\{ {{{\left[ {\begin{array}{*{20}{c}}
{{{\bf{P}}}}\\
{{{\bf{Q}}}}
\end{array}} \right]}^{\textrm{{T}}}}\left[ {\begin{array}{*{20}{c}}
{}&{\bf{H}}\\
{{{\bf{H}}^{\textrm{{T}}}}}&{}
\end{array}} \right]\left[ {\begin{array}{*{20}{c}}
{{{\bf{P}}}}\\
{{{\bf{Q}}}}
\end{array}} \right]} \right\}\\
&\quad \quad \quad \quad \quad \quad\textrm{s.t. }{{{\left[ {{{\bf{P}}}\;{{\bf{Q}}}} \right]}^{\textrm{{T}}}}\left[ {{{\bf{P}}}\;{{\bf{Q}}}} \right] = {\bf{I}}}
\end{aligned}
\end{equation}

The optimal solution of the problem (\ref{nnM}) can be solved by
\begin{equation}\label{nn5}
\frac{1}{2}\left[ {\begin{array}{*{20}{c}}
{}&{\bf{H}}\\
{{{\bf{H}}^{\textrm{{T}}}}}&{}
\end{array}} \right]\left[ {\begin{array}{*{20}{c}}
{{{\bf{P}}}}\\
{{{\bf{Q}}}}
\end{array}} \right] = \left[ {\begin{array}{*{20}{c}}
{{{\bf{P}}}}\\
{{{\bf{Q}}}}
\end{array}} \right]{\bf{\Lambda }}
\end{equation}
where ${\bf{\Lambda }}$ is diagonal matrix whose elements are composed of eigenvalues of $\frac{1}{2}\left[ {\begin{array}{*{20}{c}}
{}&{\bf{H}}\\
{{{\bf{H}}^{\textrm{{T}}}}}&{}
\end{array}} \right]$.

By simple algebra, we have
\begin{equation}\label{nn3}
\left\{ \begin{array}{l}
\frac{1}{2}{\bf{H}}{{\bf{Q}}} = {{\bf{P}}}{\bf{\Lambda }}\\
\frac{1}{2}{{\bf{H}}^{\textrm{{T}}}}{{\bf{P}}} = {{\bf{Q}}}{\bf{\Lambda }}
\end{array} \right.
\end{equation}

Then,
\begin{equation}\label{nn8}
\left\{ \begin{array}{l}
{{(\sqrt 2 } \mathord{\left/
 {\vphantom {{(\sqrt 2 } 2}} \right.
 \kern-\nulldelimiterspace} 2}{\bf{H}}{)^{\textrm{{T}}}}{{(\sqrt 2 } \mathord{\left/
 {\vphantom {{(\sqrt 2 } 2}} \right.
 \kern-\nulldelimiterspace} 2}{\bf{H}}){{\bf{Q}}} = {{\bf{Q}}}{(\sqrt 2 {\bf{\Lambda }})^2}\\
{{(\sqrt 2 } \mathord{\left/
 {\vphantom {{(\sqrt 2 } 2}} \right.
 \kern-\nulldelimiterspace} 2}{\bf{H}}){{(\sqrt 2 } \mathord{\left/
 {\vphantom {{(\sqrt 2 } 2}} \right.
 \kern-\nulldelimiterspace} 2}{\bf{H}}{)^{\textrm{{T}}}}{{\bf{P}}} = {{\bf{P}}}{(\sqrt 2 {\bf{\Lambda }})^2}
\end{array} \right.
\end{equation}

According to the model (\ref{nn8}), ${{\bf{P}}}$ and ${{\bf{Q}}}$ are composed of the leading $K$ left and right singular vectors of ${{\sqrt 2 } \mathord{\left/
 {\vphantom {{(\sqrt 2 } 2}} \right.
 \kern-\nulldelimiterspace} 2}{\bf{H}}$. Denote by $\mathbf{U}_1$ and $\mathbf{V}_1$ are the leading $K$ left and right singular vectors of $\mathbf{H}$, respectively. Thus, we have ${{\bf{Q}}} = \frac{\sqrt{2}}{2}{{\bf{V}}_1}$, ${{\bf{P}}} = \frac{\sqrt{2}}{2}{{\bf{U}}_1}$.
\end{proof}

According to Theorem~\ref{lem1}, the optimal $\mathbf{F}^{\ast}$ in Eq. (\ref{obj-jty-13}) is $\mathbf{F}^{\ast} = {{\sqrt 2 } \mathord{\left/ {\vphantom {{\sqrt 2 } 2}} \right. \kern-\nulldelimiterspace} 2}{[{\bf{U}}_1^\textrm{T}\;{\bf{V}}_1^\textrm{T}]^\textrm{T}}$, where $\mathbf{U}_1$ and $\mathbf{V}_1$ can be obtained by performing SVD on $\mathbf{H}$, it takes $\bm{\mathcal{O}}(VNM + M^2N)$ time, which scales linearly with the data size. Hence, toward large-scale clustering, due to the number of anchors $M\ll N$, it is more efficient to directly solve Eq. (\ref{obj-jty-12}) by tackling Eq. (\ref{obj-jty-13}) instead.

$\bullet$ \textbf{Solving $\mathcal{{J}}$ with fixed $\mathbf{Z}^{(v)}$ and $\textbf{F}$}. In this case, $\mathcal{{J}}$ can be solved by
\begin{equation}\label{obj-jty-14}
\begin{aligned}
{\bm{{\mathcal J}}^*} &= \mathop {\arg \min }\limits_{\bm{\mathcal {J}}} { \| \bm{{\mathcal J}} \|_{\Sp}^p} + \langle {\bm{\mathcal Y}},\bm{{\mathcal Z} - \bm{\mathcal J}} \rangle  + \frac{\mu }{2}\| {\bm{{\mathcal Z}} - \bm{{\mathcal J}}} \|_F^2\\
&= \mathop {\arg \min }\limits_{\bm{\mathcal J}} \frac{1}{\mu}{ \| \bm{{\mathcal J}} \|_{\Sp}^p} + \frac{1}{2}\| {\bm{{\mathcal Z}}  + \frac{{\bm{\mathcal Y}}}{\mu }} - \bm{{\mathcal J}}\|_F^2
\end{aligned}
\end{equation}

To solve the model (\ref{obj-jty-14}), we introduce the Theorem~\ref{theorem1}.
\begin{theorem}\label{theorem1}~\cite{TPAMI}
Suppose $\bm{{\mathcal X}} \in {\mathbb{R}^{{n_1} \times {n_2} \times {n_3}}}$, $h = \min ({n_1},{n_2})$, let $\bm{{\mathcal X}} = \bm{{\mathcal U}}*\bm{{\mathcal S}}*{\bm{{\mathcal V}}^T}$. For the following model:
\begin{equation}\label{t4-1}
\mathop {{\mathop{\rm argmin}\nolimits} }\limits_{\bm{\mathcal J}} \frac{1}{2}\| {\bm{\mathcal J}} - {\bm{{\mathcal X}}} \|_F^2 + \tau {\| \bm{{\mathcal J}} \|_{\Sp}^p}
\end{equation}
the optimal solution ${{\bm{{\mathcal J}}}^*}$ is
\begin{equation}\label{t4-2}
{{\bm{{\mathcal J}}}^{\ast}} = {\Gamma _{\tau \cdot{n_3}}}({\bm{{\mathcal X}}}) = {\bm{{\mathcal U}}}*ifft({{\bf{P}}_{\tau \cdot{n_3}}}(\overline {\bm{{\mathcal X}}} ))*{{\bm{{\mathcal V}}}^\textrm{{T}}}
\end{equation}
where, ${{\bf{P}}_{\tau \cdot{n_3}}}(\overline {\bm{\mathcal X}} )$ is a tensor, ${{\bf{P}}_{\tau \cdot{n_3}}}({\overline {{\bm{\mathcal X}}}^{(i)}})$ is the $i$-th frontal slice of ${{\bf{P}}_{\tau \cdot{n_3}}}(\overline {\bm{\mathcal X}} )$.
\end{theorem}

Thus, according to Theorem~\ref{theorem1}, the solution of the model (\ref{obj-jty-14}) is
\begin{equation}\label{o4-1}
{\bm{{\mathcal J}}^{\ast}}{\rm{ = }}{\Gamma _{\frac{1}{\mu }}}\left(\bm{{\mathcal Z}} + \frac{1}{\mu }\bm{{\mathcal Y}}\right).
\end{equation}

$\bullet$ \textbf{Solving $\mathbf{Z}^{(v)}$ with fixed $\mathcal{{J}}$ and $\textbf{F}$}. Now, the optimization w.r.t. $\mathbf{Z}^{(v)}$ in the model (\ref{obj-jty-10}) is equivalent to
\begin{equation}\label{obj-jty-19}
\begin{aligned}
{{\mathbf{Z}^{(v)}}^*} &= \mathop {\min }\limits_{\mathbf{Z}^{(v)}}\sum\limits_{v = 1}^V \left({\sqrt {\sum\limits_{i = 1}^N {\sum\limits_{j = 1}^M {\| {{\bf{x}}_i^{(v)}{\bf{ - a}}_j^{(v)}} \|_2^2{Z_{ij}^{(v)}}} } } {\bf{ + }}\alpha \| {\mathbf{{Z}}^{(v)}} \|_F^2}\right)\\
&\quad\quad\quad+\beta {tr({\bf{F}}^{\textrm{T}}{{\widetilde {\bf{L}}}_{{{\bf{Z}}}}}{\bf{F}})}+\frac{\mu}{2}\sum\limits_{v = 1}^V\left\| {{{\bf{Z}}^{(v)}}{\bf{ - }}{{\bf{J}}^{(v)}}{\bf{ + }}\frac{{{\bf{Y}}^{(v)}}}{\mu}} \right\|_F^2
\end{aligned}
\end{equation}

Due to the fact that
\begin{equation}\label{obj-jty-20}
\begin{aligned}
tr{\bf{(}}{{\bf{F}}^{\textrm{T}}}{{\bf{\tilde L}}_{\bf{Z}}}{\bf{F) = }}\sum\limits_{i = 1}^N {\sum\limits_{j = 1}^M {\left\|\frac{{{\bf{p}}_i} }{{\sqrt {{D_P}(i,i)} }}-\frac{{{\bf{q}}_i} }{{\sqrt {{D_Q}(j,j)} }} \right\|_2^2} } {Z_{ij}}
\end{aligned}
\end{equation}
thus, the model (\ref{obj-jty-20}) can be rewritten as
\begin{equation}\label{obj-jty-21}
\begin{aligned}
{{\mathbf{Z}^{(v)}}^*} &= \mathop {\min }\limits_{\mathbf{Z}^{(v)}}\sum\limits_{v = 1}^V {\sqrt {\sum\limits_{i = 1}^N {\sum\limits_{j = 1}^M {\| {{\bf{x}}_i^{(v)}{\bf{ - a}}_j^{(v)}} \|_2^2{Z_{ij}^{(v)}}} } } {\bf{ + }}\alpha \| {\mathbf{{Z}}^{(v)}} \|_F^2}\\
&\quad\quad\quad\quad+ \frac{\beta}{V}\sum\limits_{i = 1}^N {\sum\limits_{j = 1}^M {\left\|\frac{{{\bf{p}}_i} }{{\sqrt {{D_P}(i,i)} }}-\frac{{{\bf{q}}_i} }{{\sqrt {{D_Q}(j,j)} }} \right\|_2^2} } {Z_{ij}^{(v)}}\\
&\quad\quad\quad\quad+\frac{\mu}{2}\sum\limits_{v = 1}^V\left\| {{{\bf{Z}}^{(v)}}{\bf{ - }}{{\bf{J}}^{(v)}}{\bf{ + }}\frac{{{\bf{Y}}^{(v)}}}{\mu}} \right\|_F^2
\end{aligned}
\end{equation}

For convenience, denote by ${{\bf{g}}^{(v)}}=\sqrt{\sum\limits_{i = 1}^N {\sum\limits_{j = 1}^M {\| {{\bf{x}}_i^{(v)}{\bf{ - a}}_j^{(v)}} \|_2^2{Z_{ij}^{(v)}}} } }$, $d_{ij}^{(v)}{\bf{ = }}{\left\| {{\bf{x}}_i^{(v)}{\bf{ - a}}_j^{(v)}} \right\|_2^2}$, $d_{ij}^f{\bf{ = }}{\left\|\frac{{{\bf{p}}_i} }{{\sqrt {{D_P}(i,i)} }}-\frac{{{\bf{q}}_i} }{{\sqrt {{D_Q}(j,j)} }} \right\|_2^2}$, $\mathbf{E}^{(v)}={{\bf{J}}^{(v)}}\mathbf{-}\frac{{{\bf{Y}}^{(v)}}}{\mu}$. The model (\ref{obj-jty-21}) can be rewritten as
\begin{equation}\label{obj-jty-23}
\begin{aligned}
{{\mathbf{Z}^{(v)}}^*} &= \mathop {\min }\limits_{\mathbf{Z}^{(v)}}\sum\limits_{v = 1}^V \Bigg\{\sum\limits_{i = 1}^N {\sum\limits_{j = 1}^M {\left(\frac{\mu}{2} + \alpha \right)} } {{({Z_{ij}^{(v)}})}^2}\\&\quad\quad\quad\quad\quad+\left({\frac{{d_{ij}^{(v)}}}{{{\bf{g}}^{(v)}}}} - \mu E_{ij}^{(v)} + \frac{\beta}{V}d_{ij}^f\right){{{Z_{ij}^{(v)}}}}
\Bigg\}
\end{aligned}
\end{equation}
where $E_{ij}^{(v)}$ is the $(i,j)$-th element of matrix ${{\bf{E}}^{(v)}}$.

Note that the model (\ref{obj-jty-23}) is independent between different $i$, we can solve it individually for each data point $i$.  Denote by $\sigma_{i,j}^{(v)}={\frac{{d_{ij}^{(v)}}}{{{\bf{g}}^{(v)}}}} - \mu E_{ij}^{(v)} + \frac{\beta}{V}d_{ij}^f$, the model (\ref{obj-jty-23}) can be also written as
\begin{equation}\label{obj-jty-24}
\begin{aligned}
{{\mathbf{z}_i^{(v)}}^*} &= \mathop {\min }\limits_{\mathbf{z}_i^{(v)}}\left\| {{\mathbf{z}_i^{(v)}}-\frac{-\sigma_{i}^{(v)}}{\mu  + 2\alpha }} \right\|_2^2
\end{aligned}
\end{equation}

Thus, the closed-form solution is $\mathbf{z}_{i}^{( v )^{\ast}}=(\frac{-\sigma_{i}^{(v)}}{\mu  + 2\alpha }+\gamma \mathbf{1} ) _+$, where $\gamma$ is the Lagrangian multiplier~\cite{NieWJH16}.

In this paper, we utilize the same way as reported in ~\cite{ref_MLAN} to calculate adaptive parameters $\alpha $. Finally, the optimization procedure for solving the model (\ref{obj-jty-10}) is outlined in Algorithm~\ref{A1}.
\begin{algorithm}[!t]
  \caption{Multiple Graph Learning for Scalable Multi-view Clustering}
  \label{A1}
  \LinesNumbered
  \KwIn{Data matrices: $\{\mathbf{X}^{(v)}\}_{v=1}^V\in \mathbb{R}^{N \times d_v}$, the number of anchors $M$, cluster number $K$, $\lambda$.}
  \KwOut{Graph $\bm{\mathbf{Z}}$ with $K$-connected components.}
  Initialize ${{\bf{Z}}^{(v)}}$ and ${{\bf{A}}^{(v)}}$ as reported in ~\cite{SFMC}, $\mu=10^{-5}$, $\mu_{max}=10^{12}$, $\eta=1.1$, calculate $\alpha$ as reported in ~\cite{ref_MLAN}\;
  \While{not converge}
  {
  \tcp{ for $\mathbf{F}$}
  Update $\mathbf{P}$ and $\mathbf{Q}$ by solving (\ref{obj-jty-13})\;
  \tcp{ for $\bm{\mathcal{J}}$}
  Update $\bm{\mathcal{J}}$ by using (\ref{o4-1})\;
  \tcp{ for $\mathbf{Z}^{(v)}$}
  Update $\{\mathbf{Z}^{(v)}\}_{v=1}^V$ by solving (\ref{obj-jty-24})\;

  \tcp{ for variables $\mathcal{{Y}}$ and $\mu$}
  $\bm{\mathcal{Y}} := \bm{\mathcal{Y}} + \mu(\bm{\mathcal{Z}}-\bm{\mathcal{J}})$\;
  $\mu := min(\eta \mu, \mu_{max})$\;
  }
  Directly achieve the $K$ clusters based on the connectivity of the graph $\mathbf{Z}$\;
  \textbf{return:} Clustering label $\mathbf{L}$.
\end{algorithm}

\subsection{Computational Complexity Analysis}
The computational bottleneck of the proposed method only lies in constructing graphs $\mathbf{Z}^{(v)}$, and solving the subproblems for $\mathbf{Z}^{(v)}$, $\mathbf{F}$ and $\bm{\mathcal{J}}$. As for constructing graphs $\mathbf{Z}^{(v)}$, it takes $\bm{\mathcal{O}}(VNMd)$ time, where $V$, $N$, $M$ and $d$ are number of views, data points, anchors and feature dimensionality, respectively. As for the $\mathbf{Z}^{(v)}$ subproblem, it takes $\bm{\mathcal{O}}(VNMK+VNM\log(M))$ time, where $K$ is the number of clusters. As for the $\mathbf{F}$ subproblem, it takes $\bm{\mathcal{O}}(VNM + M^2N)$ time. As for the $\bm{\mathcal{J}}$ subproblem, it takes $\bm{\mathcal{O}}(VNM\log(VN)+V^2MN)$ time. Considering number of iteration $T$, and the fact $M\ll N$, $K$ is a small constant, the computational complexity of our proposed method is $\bm{\mathcal{O}}(T\cdot N(M^2+VMd) )$, which scales linearly with the data size $N$.

\section{Experiment}
\subsection {Experimental Setup}
\begin{table*}[!t]
\begin{center}
\caption{The clustering performances on MSRC-v5 and Handwritten4 datasets.}
\label{result1}
\resizebox{1.6\columnwidth}{!}{
\begin{tabular}{c!{\vrule width1.2pt}ccccccc}
\Xhline{2pt}
Dataset&\multicolumn{7}{c}{\textbf{MSRC-v5}}\\
\hline
Metric&ACC&NMI&Purity&PER&REC&F-score&ARI\\
\Xhline{1.2pt}
s-CLR ($\mathbf{X}^{(1)}$)%~\cite{NieWJH16}
&0.333$\pm$0.000&0.226$\pm$0.000&0.381$\pm$0.000&0.199$\pm$0.000&0.507$\pm$0.000&0.286$\pm$0.000&0.108$\pm$0.000\\
s-CLR ($\mathbf{X}^{(2)}$)%~\cite{NieWJH16}
&0.681$\pm$0.000&0.608$\pm$0.000&0.710$\pm$0.000&0.478$\pm$0.000&0.707$\pm$0.000&0.570$\pm$0.000&0.471$\pm$0.000\\
s-CLR ($\mathbf{X}^{(3)}$)%~\cite{NieWJH16}
&0.648$\pm$0.000&0.595$\pm$0.000&0.652$\pm$0.000&0.428$\pm$0.000&0.688$\pm$0.000&0.528$\pm$0.000&0.467$\pm$0.000\\
s-CLR ($\mathbf{X}^{(4)}$)%~\cite{NieWJH16}
&0.400$\pm$0.000&0.397$\pm$0.000&0.486$\pm$0.000&0.307$\pm$0.000&0.421$\pm$0.000&0.355$\pm$0.000&0.231$\pm$0.000\\
s-CLR ($\mathbf{X}^{(5)}$)%~\cite{NieWJH16}
&0.614$\pm$0.000&0.529$\pm$0.000&0.638$\pm$0.000&0.461$\pm$0.000&0.634$\pm$0.000&0.534$\pm$0.000&0.445$\pm$0.000\\
s-CLR-Concat%~\cite{NieWJH16}
&0.590$\pm$0.000&0.509$\pm$0.000&0.614$\pm$0.000&0.451$\pm$0.000&0.482$\pm$0.000&0.466$\pm$0.000&0.377$\pm$0.000\\
Co-reg%~\cite{KumarRD11r}
&0.635$\pm$0.007&0.578$\pm$0.006&0.659$\pm$0.006&0.511$\pm$0.008&0.535$\pm$0.007&0.522$\pm$0.007&0.425$\pm$0.030\\
SwMC%~\cite{NieLL17}
&0.776$\pm$0.000&0.774$\pm$0.000&0.805$\pm$0.000&0.687$\pm$0.000&0.831$\pm$0.000&0.752$\pm$0.000&0.708$\pm$0.000\\
MVGL%~\cite{ZhanZGW18}
&0.690$\pm$0.000&0.663$\pm$0.000&0.733$\pm$0.000&0.466$\pm$0.000&0.715$\pm$0.000&0.564$\pm$0.000&0.476$\pm$0.000\\
MVSC%~\cite{LiNHH15}
&0.794$\pm$0.075&0.672$\pm$0.058&0.756$\pm$0.071&0.585$\pm$0.091&0.779$\pm$0.035&0.664$\pm$0.062&0.600$\pm$0.079\\
RDEKM%~\cite{XuHNL17}
&0.738$\pm$0.000&0.650$\pm$0.000&0.738$\pm$0.000&0.594$\pm$0.000&0.647$\pm$0.000&0.619$\pm$0.000&0.555$\pm$0.000\\
SMSC%~\cite{HuNWL20}
&0.766$\pm$0.000&0.717$\pm$0.000&0.804$\pm$0.000&0.672$\pm$0.000&0.718$\pm$0.000&0.694$\pm$0.000&0.643$\pm$0.000\\
AMGL%~\cite{NieLL16}
&0.751$\pm$0.078&0.704$\pm$0.044&0.789$\pm$0.056&0.621$\pm$0.090&0.744$\pm$0.026&0.674$\pm$0.063&0.615$\pm$0.079\\
MLAN%~\cite{ref_MLAN}
&0.681$\pm$0.000&0.630$\pm$0.000&0.733$\pm$0.000&0.494$\pm$0.000&0.718$\pm$0.000&0.694$\pm$0.000&0.643$\pm$0.000\\
RMSC&0.762$\pm$0.040&0.663$\pm$0.026&0.769$\pm$0.030&0.640$\pm$0.030&0.660$\pm$0.034&0.650$\pm$0.031&0.592$\pm$0.036\\
CSMSC&0.758$\pm$0.007&0.735$\pm$0.010&0.793$\pm$0.008&0.736$\pm$0.014&0.673$\pm$0.008&0.703$\pm$0.010&0.653$\pm$0.012\\
SFMC%~\cite{SFMC}
&0.810$\pm$0.000&0.721$\pm$0.000&0.810$\pm$0.000&0.657$\pm$0.000&0.782$\pm$0.000&0.714$\pm$0.000&0.663$\pm$0.000\\
\textbf{Ours}&\textbf{0.933}$\pm$\textbf{0.000}&\textbf{0.893}$\pm$\textbf{0.000}&\textbf{0.933}$\pm$\textbf{0.000}&\textbf{0.848}$\pm$\textbf{0.000}&\textbf{0.885}$\pm$\textbf{0.000}&\textbf{0.866}$\pm$\textbf{0.000}&\textbf{0.844}$\pm$\textbf{0.000}\\
\Xhline{2pt}
Dataset&\multicolumn{7}{c}{\textbf{Handwritten4}}\\
\hline
Metric&ACC&NMI&Purity&PER&REC&F-score&ARI\\
\Xhline{1.2pt}
s-CLR ($\mathbf{X}^{(1)}$)%~\cite{NieWJH16}
&0.660$\pm$0.000&0.683$\pm$0.000&0.699$\pm$0.000&0.527$\pm$0.000&0.722$\pm$0.000&0.609$\pm$0.000&0.558$\pm$0.000\\
s-CLR ($\mathbf{X}^{(2)}$)%~\cite{NieWJH16}
&0.698$\pm$0.000&0.731$\pm$0.000&0.731$\pm$0.000&0.592$\pm$0.000&0.803$\pm$0.000&0.681$\pm$0.000&0.640$\pm$0.000\\
s-CLR ($\mathbf{X}^{(3)}$)%~\cite{NieWJH16}
&0.660$\pm$0.000&0.651$\pm$0.000&0.661$\pm$0.000&0.441$\pm$0.000&0.760$\pm$0.000&0.558$\pm$0.000&0.495$\pm$0.000\\
s-CLR ($\mathbf{X}^{(4)}$)%~\cite{NieWJH16}
&0.403$\pm$0.000&0.452$\pm$0.000&0.426$\pm$0.000&0.310$\pm$0.000&0.383$\pm$0.000&0.343$\pm$0.000&0.262$\pm$0.000\\
s-CLR-Concat%~\cite{NieWJH16}
&0.759$\pm$0.000&0.751$\pm$0.000&0.760$\pm$0.000&0.610$\pm$0.000&0.865$\pm$0.000&0.716$\pm$0.000&0.678$\pm$0.000\\
Co-reg%~\cite{KumarRD11r}
&0.784$\pm$0.010&0.758$\pm$0.004&0.795$\pm$0.008&0.698$\pm$0.010&0.724$\pm$0.005&0.710$\pm$0.007&0.667$\pm$0.037\\
SwMC%~\cite{NieLL17}
&0.758$\pm$0.000&0.833$\pm$0.000&0.792$\pm$0.000&0.686$\pm$0.000&0.867$\pm$0.000&0.766$\pm$0.000&0.737$\pm$0.000\\
MVGL%~\cite{ZhanZGW18}
&0.811$\pm$0.000&0.809$\pm$0.000&0.831$\pm$0.000&0.721$\pm$0.000&0.826$\pm$0.000&0.770$\pm$0.000&0.743$\pm$0.000\\
MVSC%~\cite{LiNHH15}
&0.796$\pm$0.059&0.820$\pm$0.030&0.808$\pm$0.044&0.715$\pm$0.082&0.838$\pm$0.035&0.769$\pm$0.046&0.741$\pm$0.053\\
RDEKM%~\cite{XuHNL17}
&0.805$\pm$0.000&0.803$\pm$0.000&0.842$\pm$0.000&0.714$\pm$0.000&0.806$\pm$0.000&0.757$\pm$0.000&0.728$\pm$0.000\\
SMSC%~\cite{HuNWL20}
 &0.742$\pm$0.000&0.781$\pm$0.000&0.759$\pm$0.000&0.675$\pm$0.000&0.767$\pm$0.000&0.717$\pm$0.000&0.685$\pm$0.000\\
AMGL
%~\cite{NieLL16}
 &0.704$\pm$0.045&0.762$\pm$0.040&0.732$\pm$0.042&0.591$\pm$0.081&0.781$\pm$0.022&0.670$\pm$0.060&0.628$\pm$0.070\\
MLAN%~\cite{ref_MLAN}
&0.778$\pm$0.045&0.832$\pm$0.027&0.812$\pm$0.045&0.706$\pm$0.053&0.871$\pm$0.017&0.779$\pm$0.039&0.752$\pm$0.044\\
RMSC&0.681$\pm$0.043&0.661$\pm$0.022&0.713$\pm$0.037&0.582$\pm$0.035&0.617$\pm$0.026&0.599$\pm$0.030&0.553$\pm$0.034\\
CSMSC&0.806$\pm$0.001&0.793$\pm$0.001&0.867$\pm$0.001&0.778$\pm$0.001&0.743$\pm$0.001&0.760$\pm$0.001&0.733$\pm$0.001\\
SFMC%~\cite{SFMC}
&0.853$\pm$0.000&0.871$\pm$0.000&0.873$\pm$0.000&0.775$\pm$0.000&0.910$\pm$0.000&0.837$\pm$0.000&0.817$\pm$0.000\\
\textbf{Ours}&\textbf{0.992}$\pm$\textbf{0.000}&\textbf{0.980}$\pm$\textbf{0.000}&\textbf{0.992}$\pm$\textbf{0.000}&\textbf{0.983}$\pm$\textbf{0.000}&\textbf{0.983}$\pm$\textbf{0.000}&\textbf{0.983}$\pm$\textbf{0.000}&\textbf{0.981}$\pm$\textbf{0.000}\\
%\textbf{Ours}$^\dagger$&\textbf{0.926}$\pm$\textbf{0.000}&\textbf{0.885}$\pm$\textbf{0.000}&\textbf{0.926}$\pm$\textbf{0.000}&\textbf{0.816}$\pm$\textbf{0.000}&0.869$\pm$0.000&\textbf{0.842}$\pm$\textbf{0.000}&\textbf{0.823}$\pm$\textbf{0.000}\\
\Xhline{2pt}
\end{tabular}}
\end{center}
\end{table*}

\begin{table*}[!t]
\begin{center}
\caption{The clustering performances on Mnist4 and Caltech101-20 datasets. Best results are highlighted with \textbf{bold} numbers.}
\label{result2}
\resizebox{1.6\columnwidth}{!}{
\begin{tabular}{c!{\vrule width1.2pt}ccccccc}
\Xhline{2pt}
Dataset&\multicolumn{7}{c}{\textbf{Mnist4}}\\
\hline
Metric&ACC&NMI&Purity&PER&REC&F-score&ARI\\
\Xhline{1.2pt}
s-CLR ($\mathbf{X}^{(1)}$)%~\cite{NieWJH16}
 &0.660$\pm$0.000&0.679$\pm$0.000&0.742$\pm$0.000&0.626$\pm$0.000&0.798$\pm$0.000&0.701$\pm$0.000&0.585$\pm$0.000\\
s-CLR ($\mathbf{X}^{(2)}$)%~\cite{NieWJH16}
&0.843$\pm$0.000&0.762$\pm$0.000&0.744$\pm$0.000&0.640$\pm$0.000&0.824$\pm$0.000&0.721$\pm$0.000&0.655$\pm$0.000\\
s-CLR($\mathbf{X}^{(3)}$)%~\cite{NieWJH16}
&0.743$\pm$0.000&0.661$\pm$0.000&0.744$\pm$0.000&0.642$\pm$0.000&0.827$\pm$0.000&0.723$\pm$0.000&0.657$\pm$0.000\\
s-CLR-Concat%$^\dagger$%~\cite{NieWJH16}
&0.897$\pm$0.000&0.747$\pm$0.000&0.897$\pm$0.000&0.813$\pm$0.000&0.822$\pm$0.000&0.817$\pm$0.000&0.756$\pm$0.000\\
Co-reg%~\cite{KumarRD11r}
&0.785$\pm$0.003&0.602$\pm$0.001&0.786$\pm$0.002&0.670$\pm$0.002&0.696$\pm$0.002&0.682$\pm$0.001&0.575$\pm$0.002\\
SwMC%$^\dagger$%~\cite{NieLL17}
&0.914$\pm$0.000&0.799$\pm$0.000&0.912$\pm$0.000&0.844$\pm$0.000&0.852$\pm$0.000&0.848$\pm$0.000&0.799$\pm$0.000\\
MVGL%$^\dagger$%~\cite{ZhanZGW18}
&0.912$\pm$0.000&0.785$\pm$0.000&0.910$\pm$0.000&0.795$\pm$0.000&0.804$\pm$0.000&0.800$\pm$0.000&0.733$\pm$0.000\\
MVSC%~\cite{LiNHH15}
&0.733$\pm$0.115&0.651$\pm$0.069&0.780$\pm$0.070&0.650$\pm$0.092&0.773$\pm$0.041&0.704$\pm$0.066&0.592$\pm$0.096\\
RDEKM%$^\dagger$%~\cite{XuHNL17}
&0.885$\pm$0.000&0.717$\pm$0.000&0.885$\pm$0.000&0.795$\pm$0.000&0.804$\pm$0.000&0.800$\pm$0.000&0.733$\pm$0.000\\
SMSC%$^\dagger$%~\cite{HuNWL20}
&0.913$\pm$0.000&0.789$\pm$0.000&0.913$\pm$0.000&0.843$\pm$0.000&0.850$\pm$0.000&0.846$\pm$0.000&0.795$\pm$0.000\\
AMGL%$^\dagger$%~\cite{NieLL16}
&0.910$\pm$0.000&0.785$\pm$0.000&0.910$\pm$0.000&0.836$\pm$0.000&0.843$\pm$0.000&0.840$\pm$0.000&0.786$\pm$0.000\\
MLAN%~\cite{ref_MLAN}
&0.744$\pm$0.001&0.659$\pm$0.001&0.744$\pm$0.000&0.643$\pm$0.001&\textbf{0.921}$\pm$\textbf{0.001}&0.757$\pm$0.001&0.656$\pm$0.001\\
RMSC&0.705$\pm$0.000&0.486$\pm$0.000&0.705$\pm$0.000&0.590$\pm$0.000&0.606$\pm$0.000&0.598$\pm$0.000&0.462$\pm$0.001\\
CSMSC&0.643$\pm$0.000&0.645$\pm$0.010&0.832$\pm$0.008&0.776$\pm$0.014&0.612$\pm$0.008&0.684$\pm$0.010&0.562$\pm$0.012\\
SFMC%$^\dagger$%~\cite{SFMC}
&0.917$\pm$0.000&0.801$\pm$0.000&0.917$\pm$0.000&0.846$\pm$0.000&0.855$\pm$0.000&0.852$\pm$0.000&0.802$\pm$0.000\\
\textbf{Ours}%$^\dagger$
&\textbf{0.948}$\pm$\textbf{0.000}&\textbf{0.860}$\pm$\textbf{0.000}&\textbf{0.948}$\pm$\textbf{0.000}&\textbf{0.901}$\pm$\textbf{0.000}&0.904$\pm$0.000&\textbf{0.903}$\pm$\textbf{0.000}&\textbf{0.870}$\pm$\textbf{0.000}\\
%\textbf{Ours}$^\dagger$&\textbf{0.938}$\pm$\textbf{0.000}&\textbf{0.855}$\pm$\textbf{0.000}&\textbf{0.938}$\pm$\textbf{0.000}&\textbf{0.884}$\pm$\textbf{0.000}&0.890$\pm$0.000&\textbf{0.887}$\pm$\textbf{0.000}&\textbf{0.849}$\pm$\textbf{0.000}\\
\Xhline{2pt}
Dataset&\multicolumn{7}{c}{\textbf{Caltech101-20}}\\
\hline
Metric&ACC&NMI&Purity&PER&REC&F-score&ARI\\
\Xhline{1.2pt}
s-CLR ($\mathbf{X}^{(1)}$)%~\cite{NieWJH16}
&0.390$\pm$0.000&0.174$\pm$0.000&0.450$\pm$0.000&0.195$\pm$0.000&0.720$\pm$0.000&0.307$\pm$0.000&0.069$\pm$0.000\\
s-CLR ($\mathbf{X}^{(2)}$)%~\cite{NieWJH16}
&0.387$\pm$0.000&0.238$\pm$0.000&0.468$\pm$0.000&0.177$\pm$0.000&0.501$\pm$0.000&0.261$\pm$0.000&0.027$\pm$0.000\\
s-CLR ($\mathbf{X}^{(3)}$)%~\cite{NieWJH16}
&0.323$\pm$0.000&0.206$\pm$0.000&0.422$\pm$0.000&0.169$\pm$0.000&0.515$\pm$0.000&0.254$\pm$0.000&0.013$\pm$0.000\\
s-CLR ($\mathbf{X}^{(4)}$)%~\cite{NieWJH16}
&0.442$\pm$0.000&0.269$\pm$0.000&0.492$\pm$0.000&0.198$\pm$0.000&0.745$\pm$0.000&0.313$\pm$0.000&0.076$\pm$0.000\\
s-CLR ($\mathbf{X}^{(5)}$)%~\cite{NieWJH16}
&0.414$\pm$0.000&0.284$\pm$0.000&0.479$\pm$0.000&0.205$\pm$0.000&0.763$\pm$0.000&0.324$\pm$0.000&0.091$\pm$0.000\\
s-CLR ($\mathbf{X}^{(6)}$)%~\cite{NieWJH16}
&0.358$\pm$0.000&0.244$\pm$0.000&0.452$\pm$0.000&0.172$\pm$0.000&0.611$\pm$0.000&0.268$\pm$0.000&0.019$\pm$0.000\\
s-CLR-Concat%~\cite{NieWJH16}
&0.596$\pm$0.000&0.429$\pm$0.000&0.653$\pm$0.000&0.313$\pm$0.000&0.817$\pm$0.000&0.453$\pm$0.000&0.285$\pm$0.000\\
Co-reg%~\cite{KumarRD11r}
&0.412$\pm$0.006&0.587$\pm$0.003&0.754$\pm$0.004&0.712$\pm$0.008&0.243$\pm$0.004&0.363$\pm$0.006&0.295$\pm$0.025\\
SwMC&0.599$\pm$0.000&0.493$\pm$0.000&0.700$\pm$0.000&0.509$\pm$0.000&0.625$\pm$0.000&0.431$\pm$0.000&0.265$\pm$0.000\\
MVGL%$^\dagger$%~\cite{ZhanZGW18}
&0.600$\pm$0.000&0.474$\pm$0.000&0.696$\pm$0.000&0.325$\pm$0.000&0.653$\pm$0.000&0.440$\pm$0.000&0.282$\pm$0.000\\
MVSC%~\cite{LiNHH15}
&0.595$\pm$0.000&0.613$\pm$0.000&0.717$\pm$0.000&0.542$\pm$0.000&0.546$\pm$0.000&0.541$\pm$0.000&0.451$\pm$0.000\\
RDEKM%$^\dagger$%~\cite{XuHNL17}
&0.424$\pm$0.000&0.572$\pm$0.000&0.768$\pm$0.000&\textbf{0.747}$\pm$\textbf{0.000}&0.299$\pm$0.000&0.427$\pm$0.000&0.368$\pm$0.000\\
SMSC%$^\dagger$%~\cite{HuNWL20}
&0.582$\pm$0.000&0.590$\pm$0.000&0.748$\pm$0.000&0.701$\pm$0.000&0.473$\pm$0.000&0.565$\pm$0.000&0.485$\pm$0.000\\
AMGL%~\cite{NieLL16}
&0.557$\pm$0.047&0.552$\pm$0.061&0.677$\pm$0.058&0.480$\pm$0.093&0.539$\pm$0.015&0.503$\pm$0.054&0.397$\pm$0.080\\
MLAN%~\cite{ref_MLAN}
&0.526$\pm$0.007&0.474$\pm$0.003&0.666$\pm$0.000&0.279$\pm$0.003&0.559$\pm$0.020&0.372$\pm$0.007&0.198$\pm$0.007\\
RMSC&0.385$\pm$0.024&0.512$\pm$0.012&0.742$\pm$0.013&0.692$\pm$0.038&0.231$\pm$0.019&0.346$\pm$0.026&0.288$\pm$0.027\\
CSMSC&0.474$\pm$0.037&0.648$\pm$0.011&0.563$\pm$0.031&0.290$\pm$0.034&0.730$\pm$0.037&0.415$\pm$0.039&0.356$\pm$0.040\\
SFMC%^$^\dagger$%~\cite{SFMC}
&0.642$\pm$0.000&0.595$\pm$0.000&0.748$\pm$0.000&0.586$\pm$0.000&0.677$\pm$0.000&0.628$\pm$0.000&0.461$\pm$0.000\\
\textbf{Ours}%$^\dagger$
&\textbf{0.717}$\pm$\textbf{0.000}&\textbf{0.703}$\pm$\textbf{0.000}&\textbf{0.763}$\pm$\textbf{0.000}&0.695$\pm$0.000&\textbf{0.858}$\pm$\textbf{0.000}&\textbf{0.768}$\pm$\textbf{0.000}&\textbf{0.717}$\pm$\textbf{0.000}\\
%\textbf{Ours}$^\dagger$&\textbf{0.776}$\pm$\textbf{0.000}&\textbf{0.758}$\pm$\textbf{0.000}&\textbf{0.862}$\pm$\textbf{0.000}&0.574$\pm$0.000&\textbf{0.696}$\pm$\textbf{0.000}&\textbf{0.629}$\pm$\textbf{0.000}&\textbf{0.549}$\pm$\textbf{0.000}\\
\Xhline{2pt}
\end{tabular}}
\end{center}
\end{table*}

\textbf{Datasets}: To verify the effectiveness of our proposed method, we make all the experiments on the following four datasets.
\begin{enumerate}
\item \textbf{{MSRC-v5}}~\cite{2005LOCUS} is an object dataset which has seven categories of tree, building, airplane, cow, face, car and bicycle and each of them contains 30 images. There are 5 views as reported in the paper~\cite{HuNWL20} which respectively are 24-dimension (D) CM feature, 576-D HOG feature, 512-D GIST feature, 256-D LBP feature, 254-D CENT feature.
\item \textbf{{Handwritten4}}~\cite{2017UCI} has 10 digits and 2000 images generated in the UCI machine learning repository. We choose 4 views in our experiment, which are 76-D FOU feature, 216-D FAC feature, 47-D ZER feature and 6-D MOR feature.
\item \textbf{{Mnist4}}~\cite{2012The} is a subset of MNIST with handwritten image recognition features. It contains 4000 samples of four categories, labelled as 0 to 3. In the experiment, we choose 3 views which respectively are 30-D ISO feature, 9-D LDA feature and 30-D NPE feature.
\item \textbf{{Caltech101-20}}~\cite{2007Learning} includes 20 categories with 2, 386 images. It is a subsets of Caltech101 datasets. We choose 48-D GABOR feature, 40-D WM feature, 254-D CENT feature, 1, 984-D HOG feature, 512-D GIST feature and 928-D LBP feature as 6 views in our experiment.
\end{enumerate}

\textbf{Baselines}: We choose 12 comparisons, including single-view and multi-view methods, \ie, Co-reg~\cite{2011Co}, SwMC~\cite{NieLL17}, MVGL~\cite{2017Graph}, MVSC~\cite{LiNHH15}, RDEKM~\cite{XuHNL17}, SMSC~\cite{HuNWL20}, AMGL~\cite{NieLL16}, MLAN~\cite{ref_MLAN}, RMSC~\cite{xia2014robust}, SFMC~\cite{SFMC}, CSMSC~\cite{luo2018consistent} and single-view constrained Laplacian rank (s-CLR)~\cite{2016The}.

\textbf{Metrics}: We use the 7 popular metrics to evaluate the clustering performance, including Accuracy (ACC), Normalized Mutual Information (NMI), Purity, Percision (PRE), Recall (REC), F-score and Adjusted Rand Index (ARI). Detailed introduction about the metrics can be found in ~\cite{2019Hyper}.
For each dataset, we repeat experiments 20 times independently and obtain the average value as the final result.

\subsection{Comparisons with State-of-the-art Methods}
Table~\ref{result1} and Table~\ref{result2} present the clustering results of our proposed method and baselines on four datasets. For CLR with single-view setting, s-CLR($\mathbf{X}^{(v)}$) denotes the results of CLR by employing features of $i$-th view, and s-CLR-Concat denotes the results of s-CLR on the concatenated view-features. By observing Table~\ref{result1} and Table~\ref{result2}, we can easily find the following results:
\begin{enumerate}
\item The clustering performance of single-view method, \ie, s-CLR are generally inferior to the multi-view clustering methods. In addition, the effect of s-CLR has obvious diversities in different views. The reason might be that each view contains some content of the objects that cannot be included in other views. Therefore, it is reasonable to learn a similarity matrix for each view.
\item Among the multi-view clustering methods, most of them are better than Co-reg. This is mainly because that the clustering results of Co-reg may largely depend on manually defined factors. However, in real-world situations, it is very difficult to find a suitable graph for complex data artificially.
\item Our proposed method is obviously superior to SFMC. For instance, on the MSRC-V5 dataset, our proposed method gains significant improvement around 12$\%$, 17$\%$, 12$\%$, 29$\%$, 9\%, 15\%, and 18\% in terms of ACC, NMI, Purity, PER, REC, F-score, and ARI, respectively. This is probably because that our proposed method well exploit the complementary information and spatial structure embedded in multi-view data via tensor Schatten $p$-norm, which helps to well characterize the cluster structure, while SFMC does not.
\item Particularly, despite our proposed method is an anchor-based method, the clustering performance of our proposed method is also superior to MLAN. This is because that our method takes into account significant difference of different views. Specifically, our method learns similarity graph for each view individually, and leverages tensor Schatten $p$-norm regularizer to minimize differences between similarity graphs, which also help to exploit the complementary information and spatial structure embedded in similarity graphs, while MLAN does not. These results also indicate that anchor graph can totally well encode cluster structure of data and afford efficient clustering.
\end{enumerate}

\begin{figure*}[!t]
\centering
\subfigure[Handwritten4]{
\begin{minipage}[t]{0.23\linewidth}
\centering
\includegraphics[width=4.3cm]{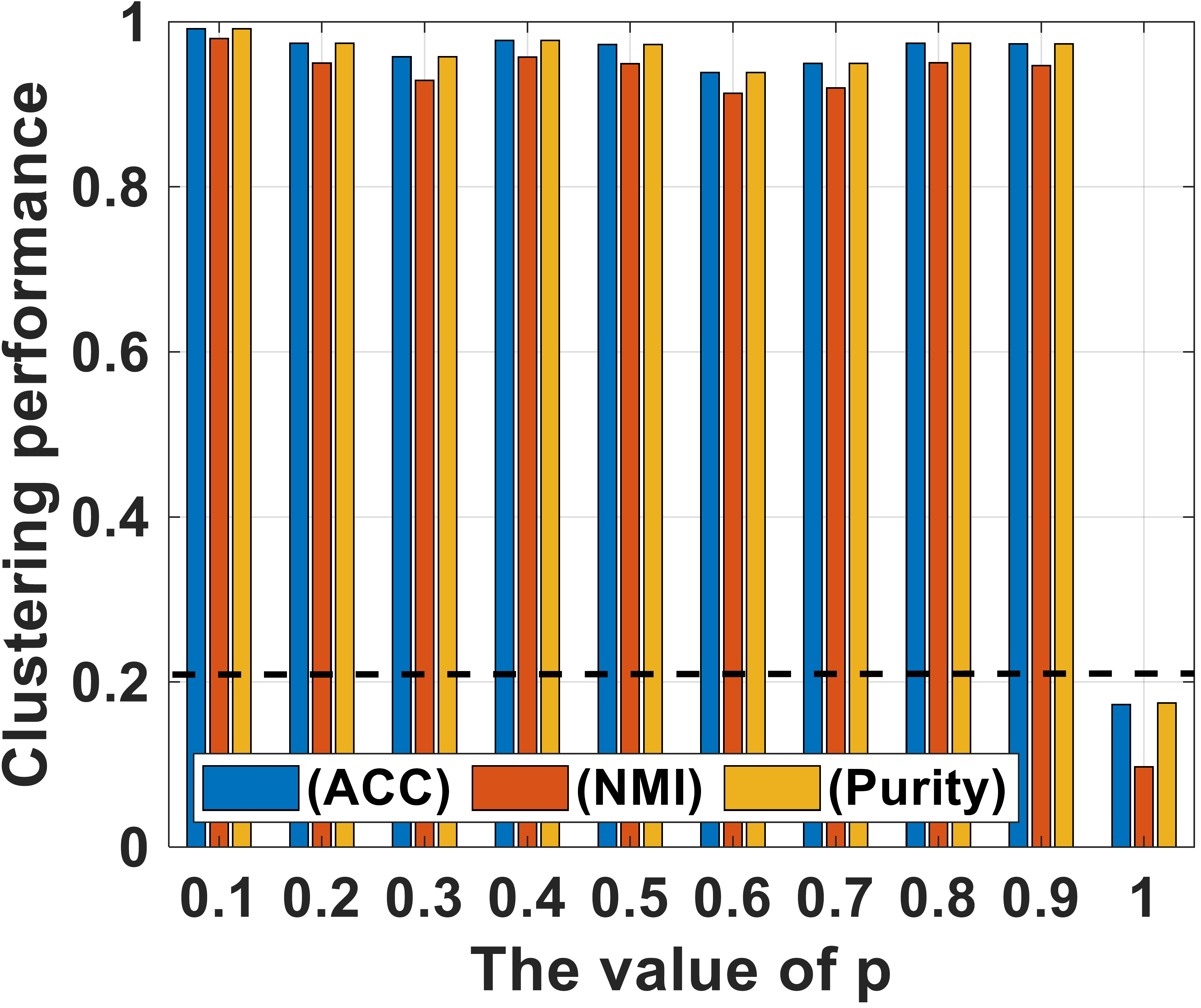}
\end{minipage}
}
\subfigure[Mnist4]{
\begin{minipage}[t]{0.23\linewidth}
\centering
\includegraphics[width=4.3cm]{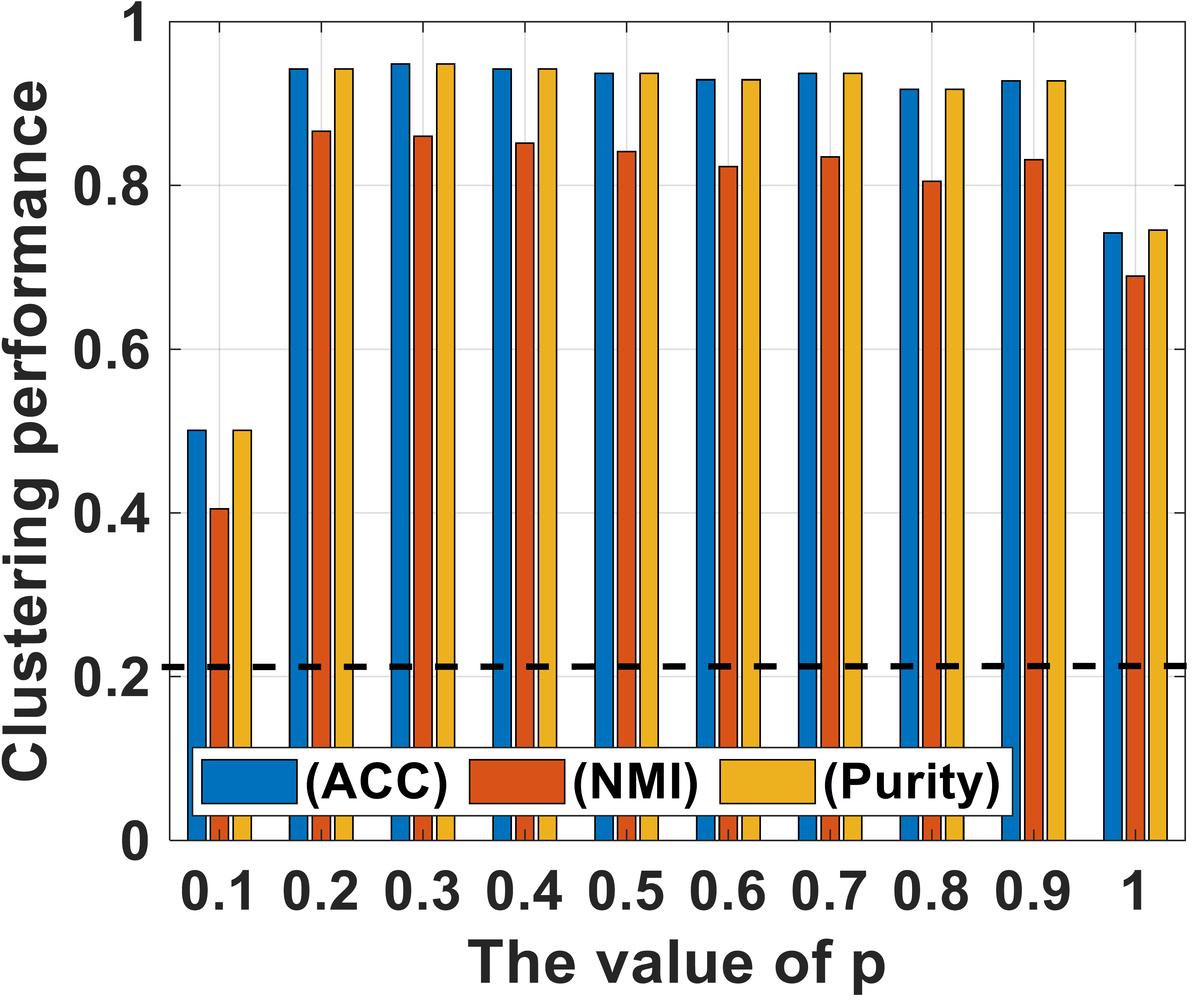}
\end{minipage}
}
\subfigure[MSRC-v5]{
\begin{minipage}[t]{0.23\linewidth}
\centering
\includegraphics[width=4.3cm]{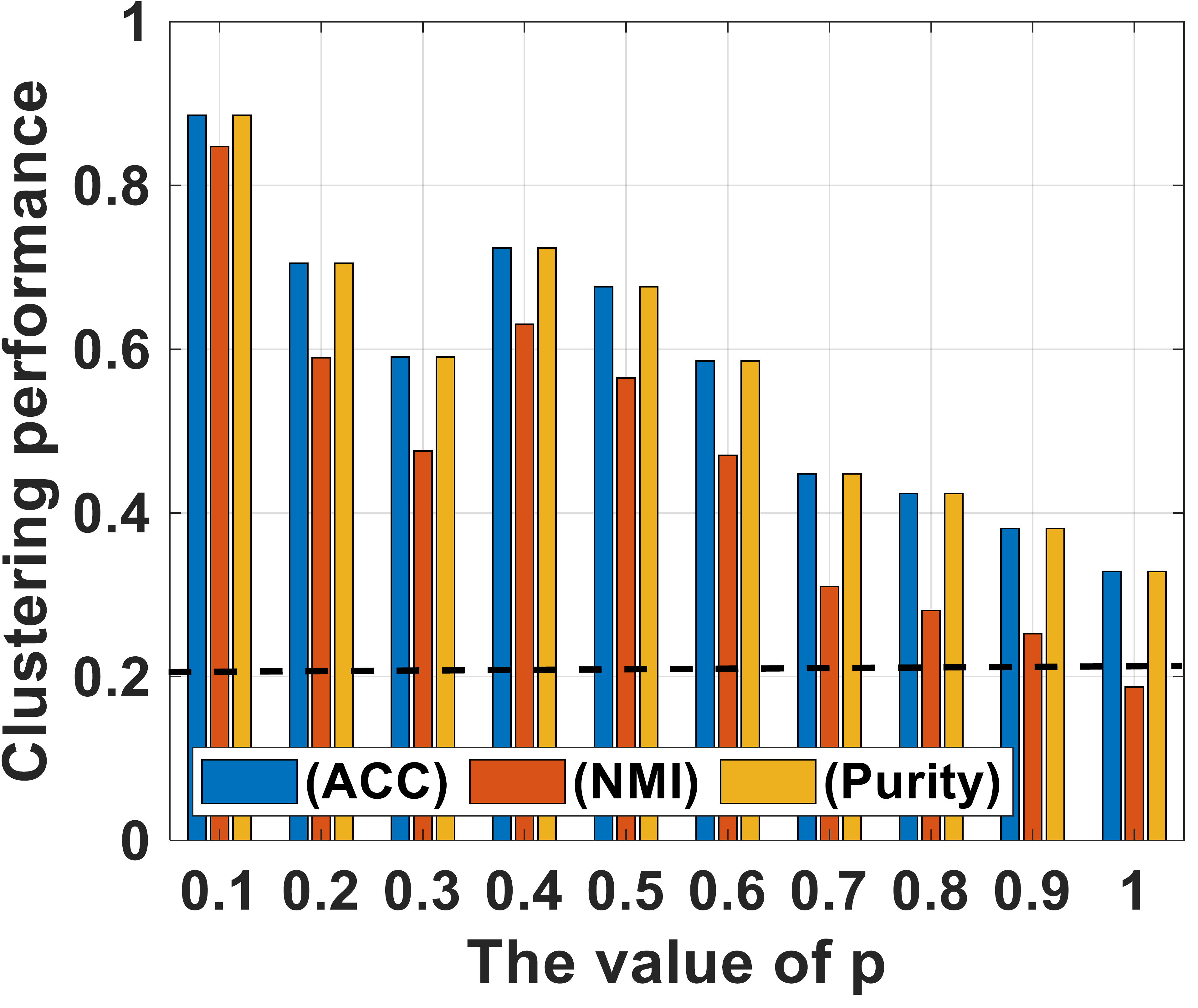}
\end{minipage}
}
\subfigure[Caltech101-20]{
\begin{minipage}[t]{0.23\linewidth}
\centering
\includegraphics[width=4.3cm]{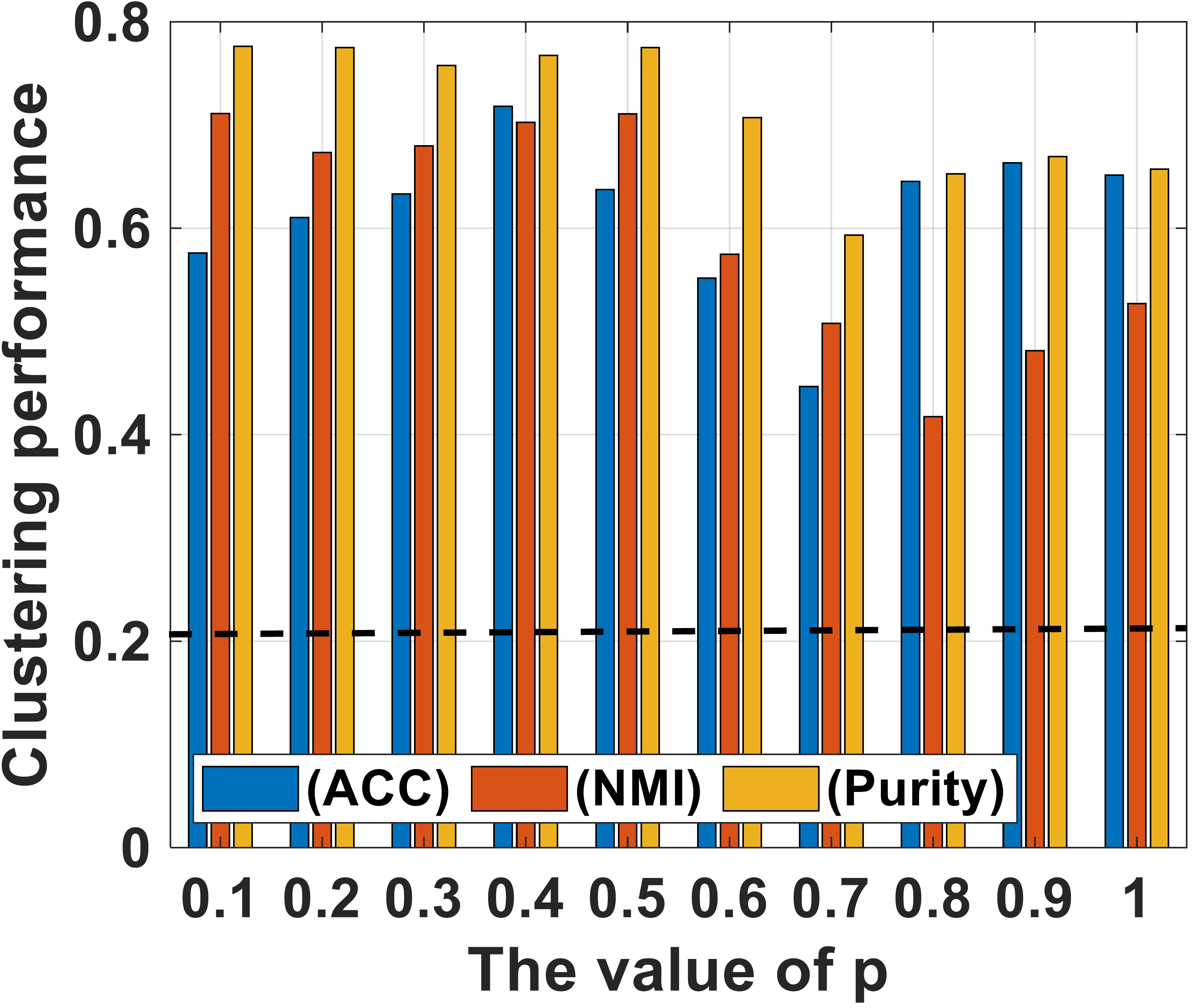}
\end{minipage}
}
\centering
\caption{The clustering performances of our method with the varying value of $p$ on Handwritten4, Mnist4, MSRC-v5 and Caltech101-20 datasets.}
\label{p-value}
\end{figure*}

\begin{figure*}[!t]
\centering
\subfigure[MSRC-v5]{
\begin{minipage}[t]{0.23\linewidth}
\centering
\includegraphics[width=4.3cm]{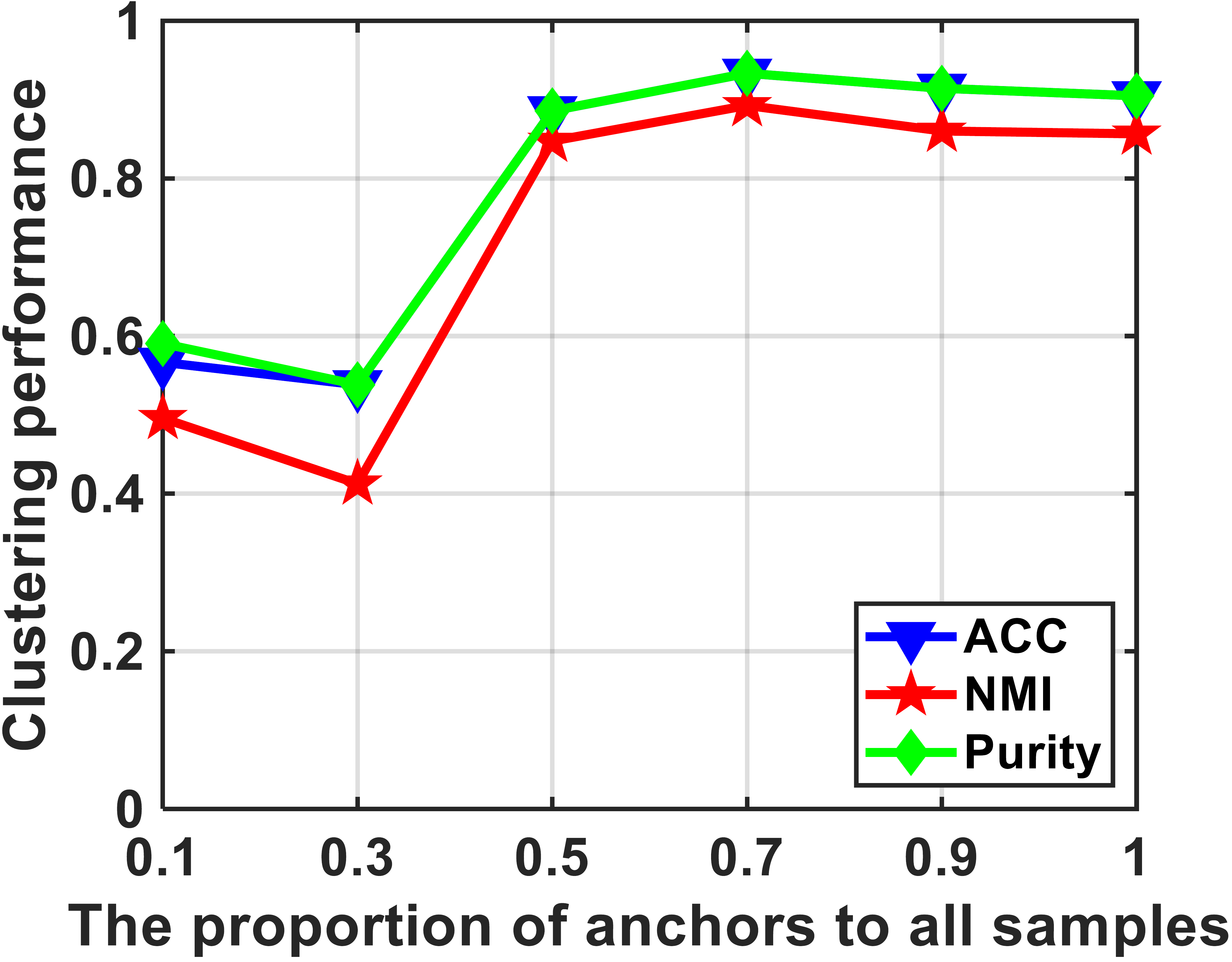}
\end{minipage}
}
\centering
\subfigure[Handwritten4]{
\begin{minipage}[t]{0.23\linewidth}
\centering
\includegraphics[width=4.3cm]{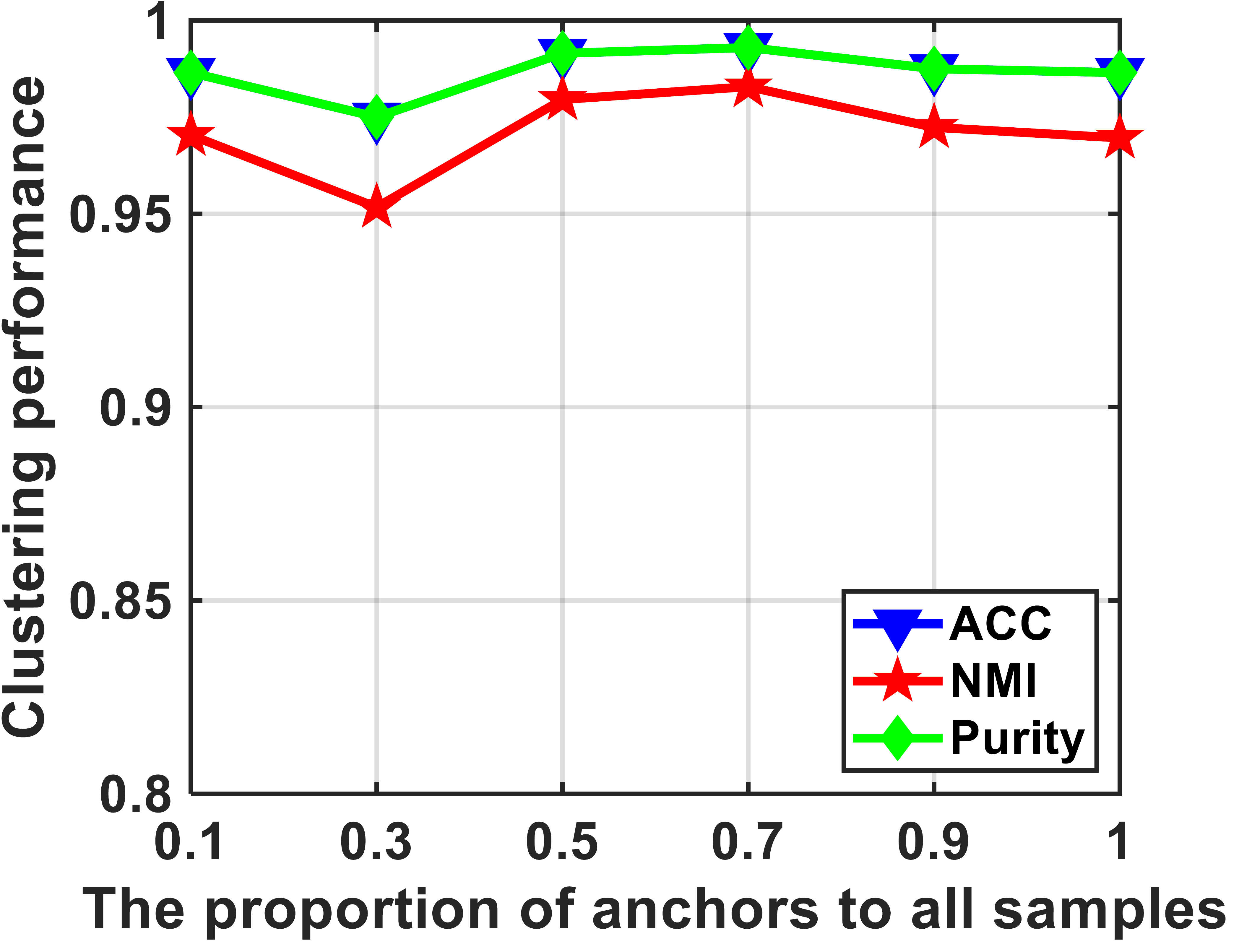}
\end{minipage}
}
\subfigure[Caltech101-20]{
\begin{minipage}[t]{0.23\linewidth}
\centering
\includegraphics[width=4.3cm]{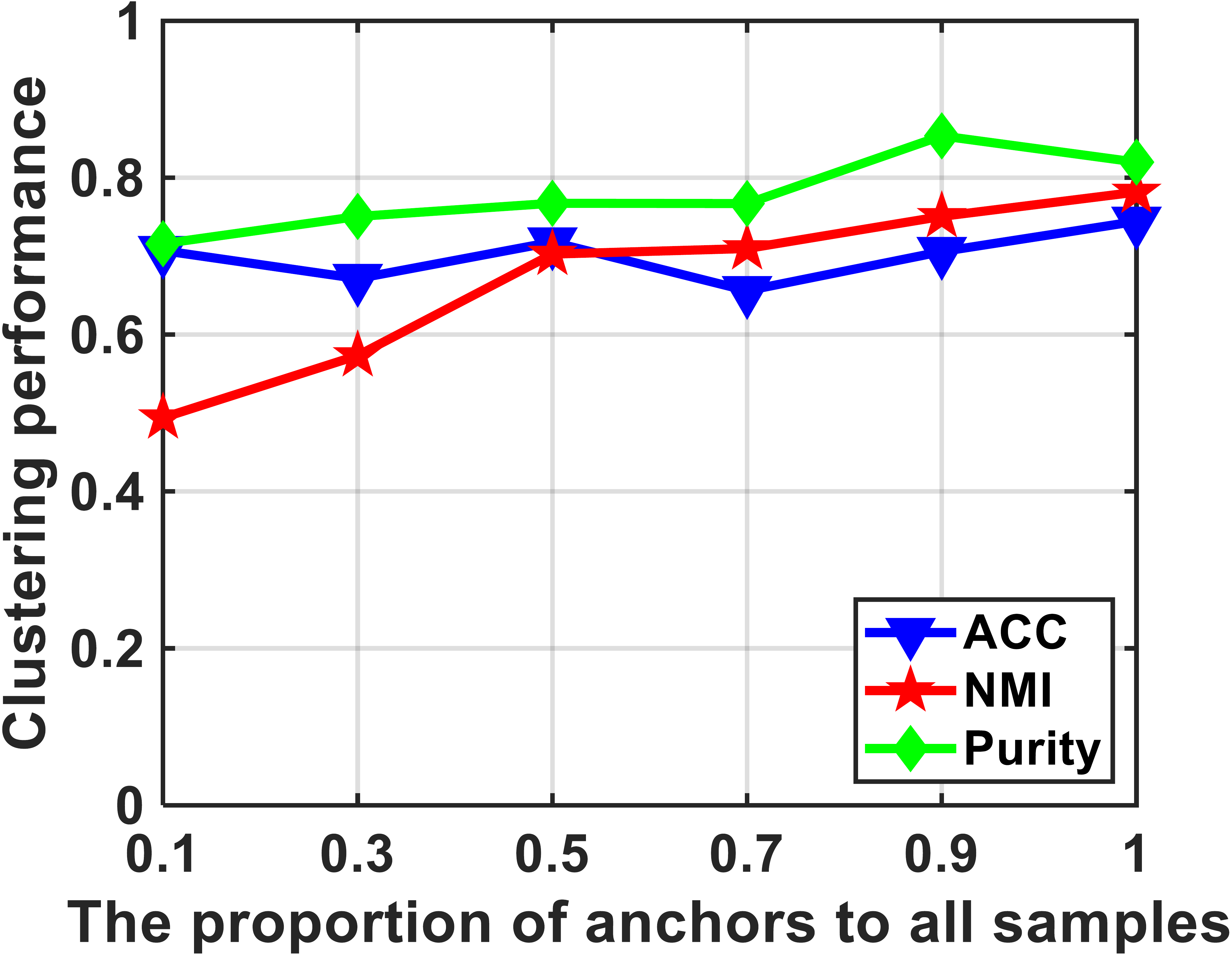}
\end{minipage}
}
\subfigure[Mnist4]{
\begin{minipage}[t]{0.23\linewidth}
\centering
\includegraphics[width=4.3cm]{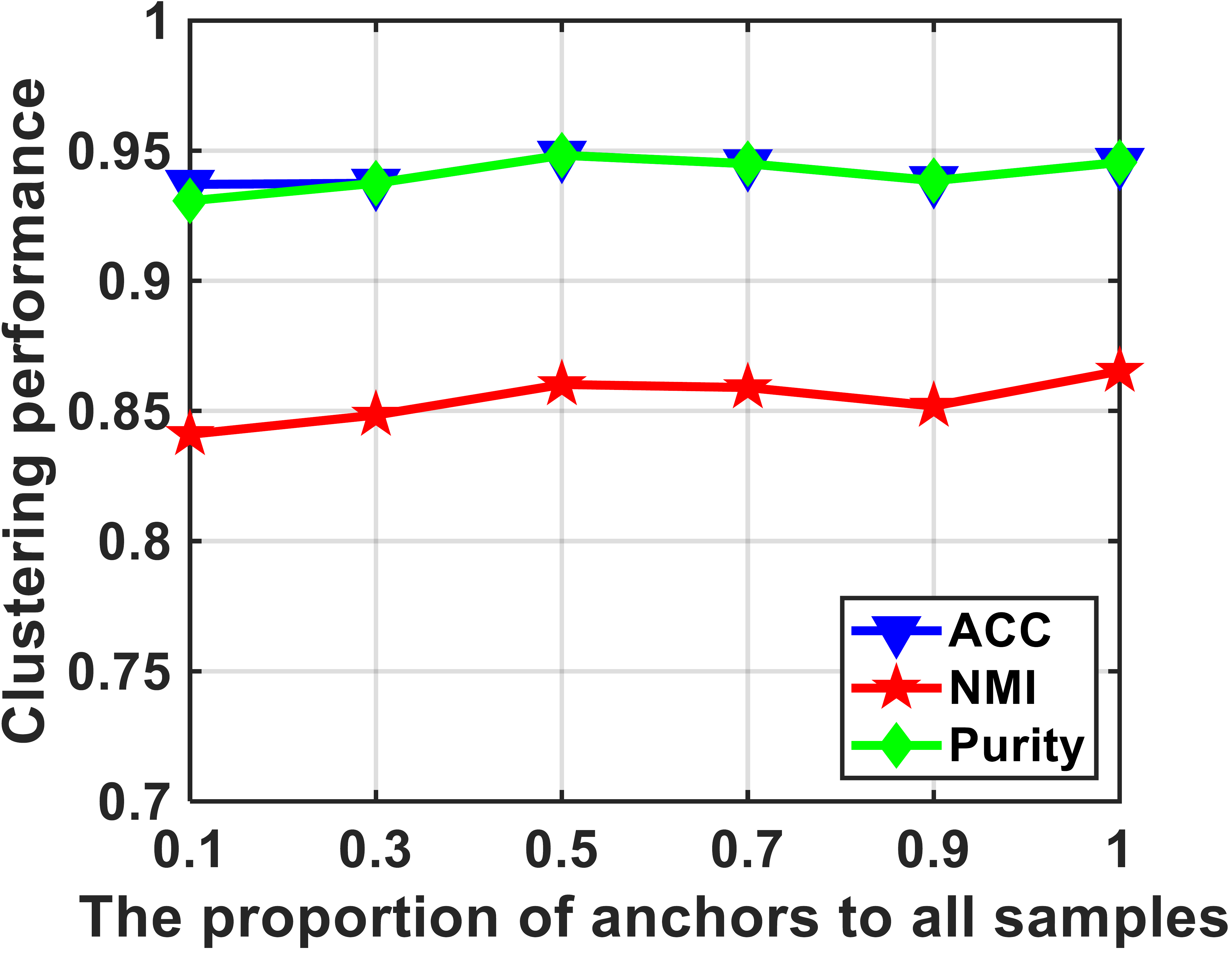}
\end{minipage}
}
\centering
\caption{The performances of our method with varying the number of anchor points on MSRC-v5 , Handwritten4, Caltech101-20 and Mnist4 datasets.}
\label{a-value}
\end{figure*}

\begin{figure}[!t]
\begin{center}
\includegraphics[width=1.0\linewidth]{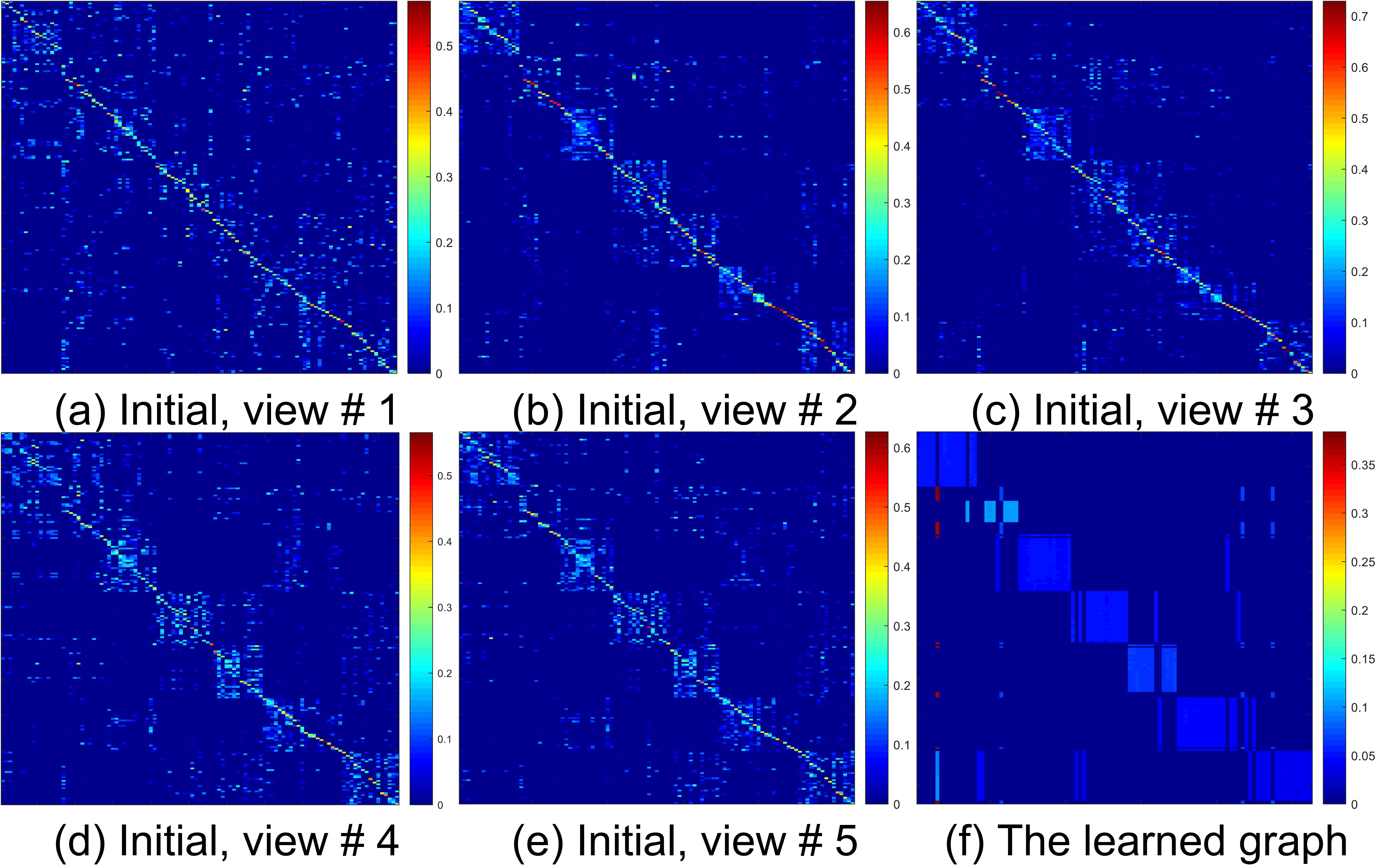}
\end{center}
   \caption{The graphs visualizations on MSRC-v5 dataset.}
\label{visual}
\end{figure}
\subsection{Further Evaluation}
In this subsection, we further evaluate the effect of the parameter $p$, the number of anchors on the clustering performance. We also visualize the learned  graph.

\textbf{Effect of parameter $p$.}
We evaluate the effect of $p$ by comparing the clustering performance on four datasets with different $p$ value, as shown in Figure~\ref{p-value}. In our experiments, we adjust the parameter $p$ from 0.1 to 1 with an interval of 0.1. From the comparison results in terms of ACC, NMI and Purity, the value of $p$ truly influence the clustering performance. To be specific, on Caltech101-20 and Mnist4 datasets, our method achieves the best performance when $p$=0.4 and $p$=0.3. On MSRC-v5 and Handwritten4 datasets, our method achieves the best performance when $p$=0.1. This is probably because that $p$ exploits the significant difference between singular values. Another reason may be that tensor Schatten $p$-norm makes the rank of the learned view-consensus graph well approximate the target rank.

\textbf{Effect of the number of anchors.} We verify the influence of anchors number on clustering performance. To this end, in the experiment, we adjust the proportion of anchor points from 0.1 to 1 with an interval of 0.2. From the Figure~\ref{a-value}, we can see that, our proposed method achieves the best performance when the proportion of anchor points is 0.5, 0.5, 0.7, and 1.0 on Handwritten4, Mnist4, MSRC-v5, and Caltech 101-20 datasets, respectively. Moreover, the curves in Figure~\ref{a-value} are not monotonously increasing w.r.t. anchor proportion, which indicates that it is not necessary to use many anchors in clustering. Therefore, we set the anchor ratio to 0.5 in the experiment in Tables~\ref{result1} and~\ref{result2}.

\textbf{Graph Visualization.} We visualize the initial graphs and the learned graph of our method on MSRC-v5 dataset in Figure~\ref{visual}, where (a) - (e) are initial graphs corresponding to five views, (f) is the shared graph. It can be seen that the connected components in the initial graphs of all five views are not clear. By employing our proposed method, we can observe that the learned graph has exact 7-connected components. These visualizations once again demonstrate that our proposed method well characterizes the cluster structure.

\section{Conclusion}
In this paper, we propose a multiple graph learning model for scalable multi-view clustering. Our method learns similarity graph for each view individually, and utilize tensor Schatten $p$-norm regularizer to minimize differences between similarity graphs, which help to exploit the complementary information and spatial structure embedded in similarity graphs of different views. Meanwhile, we construct a hidden and tractable large graph by anchor graph for each view to replace the full-size graph, thus our method is time-economical. By introducing the connectivity constraint, our method learns a graph with $K$-connected components, and can directly obtain the clustering without post-processing. Extensive experimental results indicate that our method is superior to the state-of-the-art methods.

\ifCLASSOPTIONcaptionsoff
  \newpage
\fi

{\small
\bibliographystyle{ieeetr}
\bibliography{egbib}
}

\end{document}